\RequirePackage{fix-cm}
\documentclass{externalized}
\usepackage{microtype}
\usepackage{graphicx}
\usepackage{hyperref}
\usepackage{url}
\usepackage{booktabs}
\usepackage{amsfonts}
\usepackage{amsmath}
\usepackage{amssymb}
\usepackage{amsthm}
\usepackage{mathtools}
\usepackage{enumitem}
\usepackage{multirow}
\usepackage{subcaption}
\usepackage{xcolor}
\usepackage{listings}
\usepackage{natbib}
\usepackage{multicol}
\usepackage{pifont}
\usepackage{xspace}
\usepackage{wrapfig}
\usepackage{float}
\usepackage{tabularx}
\usepackage{ragged2e}
\usepackage{longtable}
\usepackage{colortbl}

\DeclareTextFontCommand{\tabletextbf}{\fontseries{bx}\selectfont}
\AtBeginEnvironment{tabular}{\let\textbf\tabletextbf}
\AtBeginEnvironment{tabular*}{\let\textbf\tabletextbf}
\AtBeginEnvironment{tabularx}{\let\textbf\tabletextbf}
\AtBeginEnvironment{longtable}{\let\textbf\tabletextbf}

\providecommand{\noopsort}[1]{}
\AtBeginDocument{%
  \providecommand\BibTeX{{%
    \normalfont B\kern-0.5em{\scshape i\kern-0.25em b}\kern-0.8em\TeX}}}

\makeatletter
\DeclareRobustCommand\onedot{\futurelet\@let@token\@onedot}
\def\@onedot{\ifx\@let@token.\else.\null\fi}

\makeatother

\definecolor{codegreen}{rgb}{0,0.6,0}
\definecolor{codegray}{rgb}{0.5,0.5,0.5}
\definecolor{codepurple}{rgb}{0.58,0,0.82}
\definecolor{backcolour}{rgb}{0.95,0.95,0.92}

\lstdefinestyle{mystyle}{
  backgroundcolor=\color{backcolour},
  commentstyle=\color{codegreen},
  keywordstyle=\color{magenta},
  numberstyle=\tiny\color{codegray},
  stringstyle=\color{codepurple},
  basicstyle=\ttfamily\footnotesize,
  breakatwhitespace=false,
  breaklines=true,
  captionpos=b,
  keepspaces=true,
  numbers=left,
  numbersep=5pt,
  showspaces=false,
  showstringspaces=false,
  showtabs=false,
  tabsize=2
}

\usepackage{needspace}
\usepackage{fontawesome5}
\usepackage{xurl}
\usepackage{adjustbox}
\usepackage{array}
\usepackage{fvextra}
\usepackage{capt-of}
\usepackage{wrapfig}
\usepackage{needspace}

\definecolor{sectionblue}{HTML}{315F8C}
\definecolor{abstractlink}{HTML}{315F8C}
\definecolor{scopepromptblue}{HTML}{506C8A}
\definecolor{scopepromptteal}{HTML}{317F79}
\definecolor{scopepromptpurple}{HTML}{796590}
\definecolor{scopebestcol}{HTML}{E6DAEF}
\definecolor{scopesecondcol}{HTML}{E2ECF1}
\definecolor{scopeband}{HTML}{F2F2F2}
\definecolor{scopeablationfull}{HTML}{EAEAF8}
\definecolor{scopeoursrow}{HTML}{F7F4FA}
\definecolor{scopecitecol}{HTML}{2276BA}
\newcommand{\scopebest}[1]{\cellcolor{scopebestcol}\textbf{#1}}
\newcommand{\scopesecond}[1]{\cellcolor{scopesecondcol}\underline{#1}}
\newcommand{\scopetablecite}[1]{\textcolor{scopecitecol}{(\citeauthor{#1},~\citeyear{#1})}}
\newtheorem{scopeproposition}{Proposition}[section]
\newtheorem{scopecorollary}[scopeproposition]{Corollary}
\newtcolorbox{scopepromptbox}[2]{breakable,title={#1},colback=#2!3!white,
  colframe=#2!75!white,colbacktitle=#2!12!white,coltitle=black!85,
  fonttitle=\small\bfseries\sffamily,boxrule=0.45pt,arc=1.2mm,
  boxsep=0.6mm,left=1.5mm,right=1.5mm,top=1.1mm,bottom=1.1mm,
  before skip=7pt,after skip=7pt,fontupper=\footnotesize,parbox=false}

\usepackage[most]{tcolorbox}
\definecolor{scopeblue}{HTML}{397CB5}

\newtcolorbox{scopequestion}{
    enhanced,
    colback=scopeblue!5,
    frame hidden,
    borderline west={2.5pt}{0pt}{scopeblue},
    sharp corners,
    boxsep=0pt,
    left=9pt,
    right=8pt,
    top=7pt,
    bottom=7pt,
    before skip=7pt,
    after skip=9pt
}

\definecolor{scopeS}{HTML}{668DD0}
\definecolor{scopeC}{HTML}{8D78C8}
\definecolor{scopeO}{HTML}{C483BA}
\definecolor{scopeP}{HTML}{C8956F}
\definecolor{scopeE}{HTML}{78A895}

\newcommand{\scopewordmark}{%
  \textcolor{scopeS}{S}%
  \textcolor{scopeC}{C}%
  \textcolor{scopeO}{O}%
  \textcolor{scopeP}{P}%
  \textcolor{scopeE}{E}%
}
\renewcommand{\institutionlogosleft}{}
\title{%
    \texorpdfstring{%
        \raisebox{-0.20\height}{%
            \includegraphics[height=1.25em]{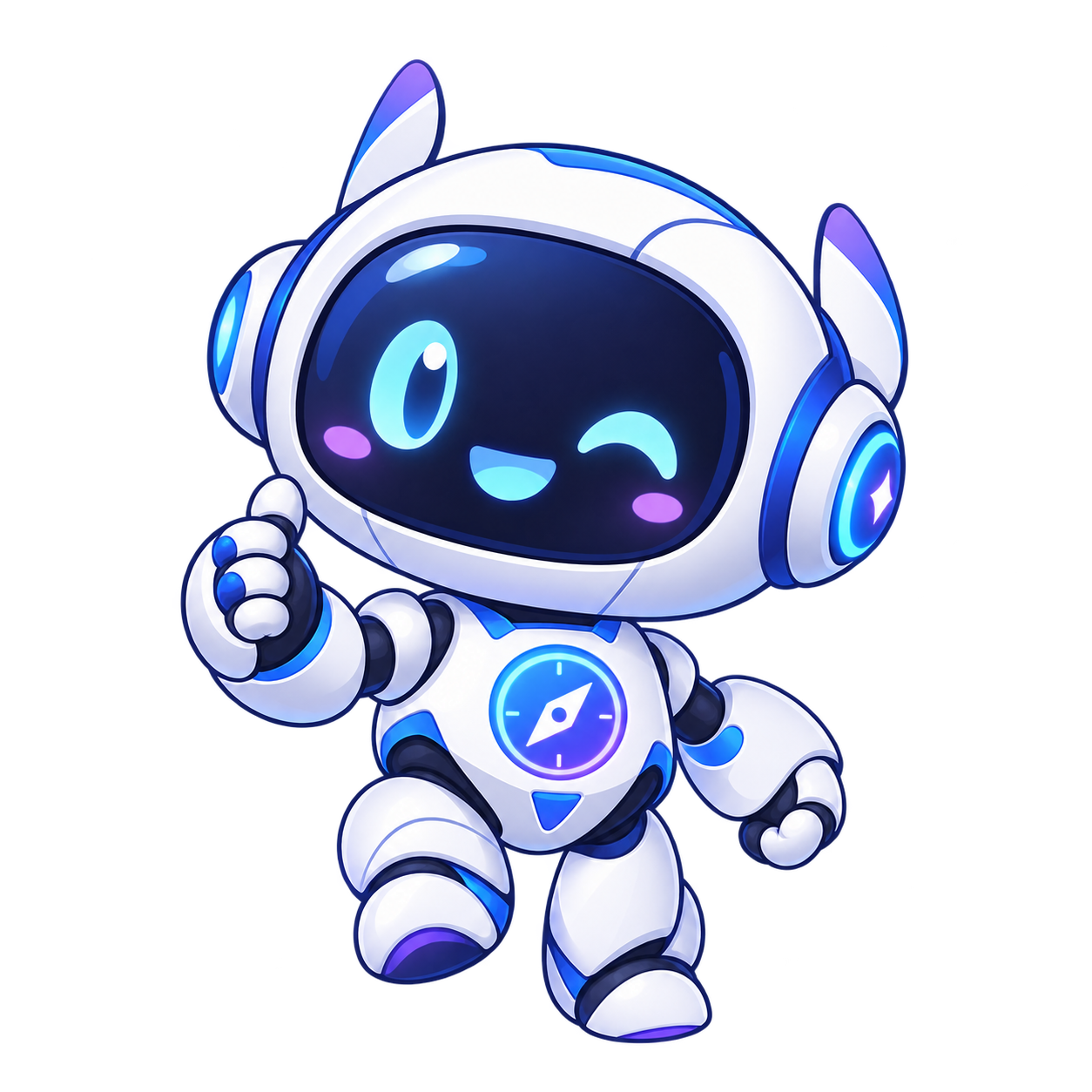}%
        }%
        \hspace{0.18em}%
    }{}%
    \scopewordmark: Selective Critical-Step On-Policy Self-Distillation for Agentic Reinforcement Learning%
    \texorpdfstring{\vspace{-3mm}}{}%
}
\author[1,*]{Qirui Liu}
\author[1,*]{Wei Cao}
\author[1]{Yu Mi}
\author[1]{Yichen Sun}
\author[1]{Xuekun Zhang}
\author[2]{Lei Liu}
\author[2]{Yue Shen}
\author[1,\dagger]{Zhixuan Chu}

\affiliation[1]{Zhejiang University}
\affiliation[2]{Ant Group}

\contribution[*]{Equal contribution}
\contribution[\dagger]{Corresponding author}
\contribution[]{\email{zhixuanchu@zju.edu.cn}}

\abstract{%
Outcome-based reinforcement learning (RL) improves long-horizon
LLM agents through task-level feedback, but provides limited
supervision over intermediate decisions. Our analysis reveals
an \emph{outcome--step mismatch}: successful trajectories
can contain erroneous steps, while failed trajectories may
include correct decisions. To provide finer-grained guidance,
recent methods incorporate hindsight skills into on-policy
self-distillation. However, we further find that applying
trajectory-level skills across all steps can introduce
off-target distillation signals at decisions unrelated to
the intended corrections. Motivated by these observations,
we introduce \textbf{SCOPE}, a framework for
\emph{Selective Critical-Step On-Policy Self-Distillation}.
SCOPE first equips the policy with error-localization and
skill-generation capabilities through lightweight supervised
initialization. During online RL, the policy analyzes its
own trajectories, generates step-specific corrective skills,
and distills their guidance exclusively at selected error
steps. A bounded signed distillation coefficient stabilizes
auxiliary updates while preserving both reinforcement and
suppression, and a success-relative pairwise length reward
encourages efficient task completion. Experiments across three smaller backbones on ALFWorld and WebShop demonstrate consistent improvements. With Qwen2.5-3B, SCOPE achieves 97.7\% and
85.2\% success, outperforming SEED by 5.9 and 6.3 percentage
points, respectively. 
On Qwen3-32B, SCOPE achieves 99.2\% ALFWorld
success versus 85.2\% for GRPO, while reducing
average interactions by 65.0\%.
\par\vspace{1.8mm}
\noindent{\color{abstractlink}\faGlobe}\hspace{0.5em}
\textbf{Project Page:}\enspace
\href{https://neuqrui.github.io/SCOPE-page}{\texttt{https://neuqrui.github.io/SCOPE-page}}
}

\begin{document}
\raggedbottom
\vspace*{-10mm}
\maketitle

\setlength{\parskip}{4pt plus 1pt minus 1pt}
\begin{figure}[H]
    \centering
    \includegraphics[width=\textwidth]{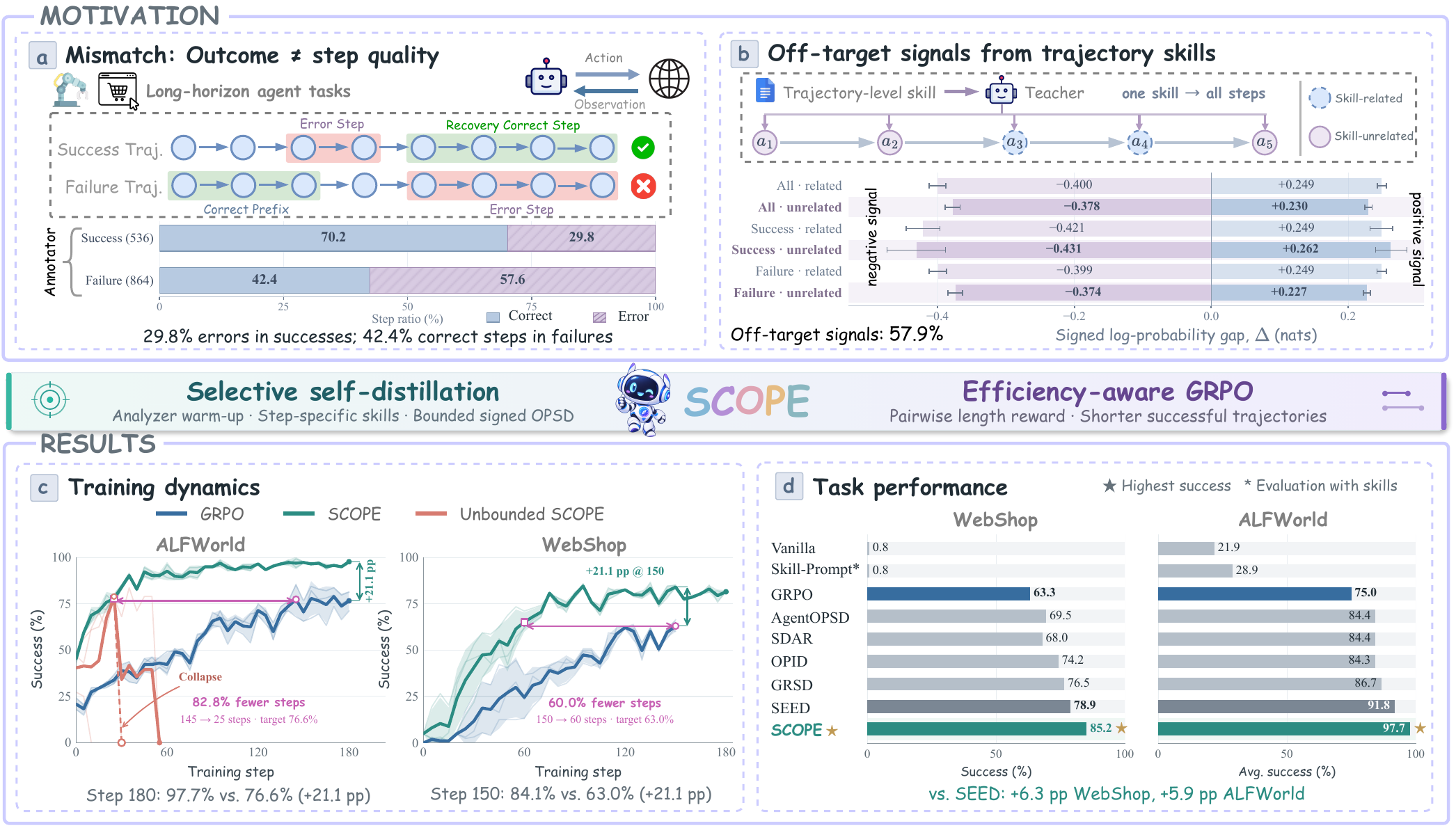}
    % \vspace{-1em}
    \caption{\textbf{Motivation and results of SCOPE.}
(a) Outcome--step mismatch.
(b) Off-target skill signals.
(c) Training dynamics (mean $\pm$ std. across seeds);
Unbounded SCOPE uses raw-gap OPSD.
(d) Qwen2.5-3B results at step 150.}
    \label{fig:intro_overview}
\end{figure}

\section{Introduction}
\label{sec:introduction}

LLM agents solve tasks through sequences of tool calls and environment
interactions~\citep{yao2022react,schick2023toolformer,liu2024agentbench}.
Reinforcement learning (RL) improves these decisions using task
feedback~\citep{qian2026toolrl,wang2025ragen,jin2025search}.
In particular, GRPO~\citep{shao2024deepseekmath} converts the outcomes
of sampled trajectories into relative advantages, providing a scalable
training signal without step-level annotations~\citep{luo2025agent,wu2026spark,wang2025ragen}.
However, an outcome reward reveals whether a task was completed,
with little indication of which intermediate decisions were sound
or required correction.

This distinction matters because \textbf{trajectory success and step
quality do not always agree}. A successful agent may recover from
invalid or mistaken actions, while a failed agent may make useful
progress before a consequential error. We examine this
\emph{outcome--step mismatch} in annotated WebShop trajectories.
Figure~\ref{fig:intro_overview}(a) shows 29.8\% erroneous steps in the
successful group and 42.4\% correct steps in the failed group.
These observations motivate identifying errors independently of the
final outcome. When an outcome-level advantage is shared across a
trajectory, it cannot distinguish these local differences in decision
quality~\citep{feng2026group,cheng2026beyond,liu2026deshortcut}.

To compensate for coarse outcome feedback, recent methods introduce
hindsight-based supervision to provide additional guidance on agent
decisions. On-policy self-distillation
(OPSD)~\citep{zhao2026self} enables such guidance: a teacher branch
receives privileged context and scores actions sampled by the ordinary
policy~\citep{agarwal2024policy,liu2026eopsa}.
Trajectory-derived skills summarize useful
hindsight~\citep{andrychowicz2017hindsight}, and methods such
as SEED train the policy to generate this guidance from its own
experience~\citep{wang2026skill,wu2026seed}.
However, \textbf{privileged skill context can induce off-target
probability shifts at steps unrelated to the skill}. A trajectory-level
skill may address a particular mistake or subgoal, yet conditioning
the teacher on it throughout the trajectory can also change token
probabilities at other steps.
EOPSA~\citep{liu2026eopsa} highlights a related issue:
privileged prompts can introduce task-irrelevant discrepancies
that dilute useful distillation signals.
In our agentic setting,
Figure~\ref{fig:intro_overview}(b) shows that skill-unrelated steps
account for 57.9\% of nonzero sampled-token gaps. Among these
off-target signals, 53.9\% are negative and 46.1\% are positive.
Distilling across all steps assigns auxiliary updates to these
probability shifts even where the skill offers no correction for
the current decision.

Together, these observations raise a central question:

\begin{scopequestion}
\textbf{\textcolor{scopeblue}{Key Question.}}
How can an agent identify critical errors independently
of trajectory outcomes and apply corrective guidance
only at relevant steps?
\end{scopequestion}

\noindent
We introduce \textbf{SCOPE}, a framework for
\emph{Selective Critical-Step On-Policy Self-Distillation}.
A lightweight supervised warm-up teaches the policy to locate
errors and generate corrective skills. During online RL, the policy
analyzes its own trajectories to identify critical errors
in both successful and failed episodes and generates
corresponding step-specific skills. A skill-conditioned
self-teacher then provides corrective supervision
exclusively at the selected steps, aligning distillation
with the decisions that require correction.

During training, however, we observe occasional abrupt
performance collapses with unbounded OPSD updates,
as illustrated in Figure~\ref{fig:intro_overview}(c).
Large teacher--student log-probability gaps can amplify
auxiliary updates and destabilize optimization.
To mitigate this instability, SCOPE employs a
\emph{bounded signed} distillation coefficient that
limits extreme updates while preserving both
reinforcement and suppression.
Moreover, successful trajectories may still contain
redundant interactions, motivating a success-relative
pairwise length reward to encourage efficient task
completion. Together, these designs enable stable,
targeted, and efficiency-aware agentic RL.

Experiments on ALFWorld and WebShop across three
backbones demonstrate consistent improvements.
With Qwen2.5-3B, SCOPE achieves 97.7\% and 85.2\%
success, outperforming SEED by 5.9 and 6.3
percentage points, respectively. On unseen ALFWorld
tasks, SCOPE achieves 89.6\% success (+23.2 points)
with 46.7\% fewer interactions. Extending to
Qwen3-32B, SCOPE reaches 99.2\% ALFWorld success
(+14.0 points over GRPO) with 65.0\% fewer
interactions.

In summary, our contributions are fourfold:
\begin{itemize}[
    leftmargin=1.2em,
    labelsep=0.45em,
    itemsep=4pt,
    topsep=4pt,
    parsep=0pt
]
    \item We reveal two limitations of trajectory-level
    supervision in agentic RL: outcome--step mismatch,
    where trajectory success does not reliably reflect
    individual decision quality, and off-target
    distillation signals induced by privileged skills
    at unrelated steps.

    \item We propose \textbf{SCOPE}, a selective on-policy
    self-distillation framework that bootstraps
    critical-step analysis from supervised warm-up,
    generates step-specific corrective skills from
    online trajectories, and distills their guidance
    exclusively at identified error steps.

    \item We introduce a bounded signed distillation
    coefficient to mitigate unstable auxiliary updates
    while preserving both positive and negative
    corrections, together with a success-relative
    pairwise length reward to encourage efficient
    task completion.

    \item We validate SCOPE on ALFWorld and WebShop
    across three backbones and further extend it to
    Qwen3-32B. On unseen ALFWorld tasks, SCOPE
    improves success by 23.2 points with 46.7\%
    fewer interactions. With Qwen3-32B, it achieves
    99.2\% seen and 94.0\% unseen ALFWorld success.

\end{itemize}

\section{Related Work}
\label{sec:related_work}

\subsection{Reinforcement Learning for LLM Agents}

Reinforcement learning trains LLM agents through feedback from tool
use and environment interaction~\citep{wu2026atlas,luo2025large,yan2026odysseyarena}.
ToolRL~\citep{qian2026toolrl},
RAGEN~\citep{wang2025ragen}, and Search-R1~\citep{jin2025search}
apply this approach to tool selection, interactive planning, and
information seeking. GRPO~\citep{shao2024deepseekmath} provides a
common optimization backbone by comparing trajectory outcomes within
a sampled group. However, trajectory-level advantages offer limited
resolution over intermediate decisions in long interactions.
GiGPO~\citep{feng2026group} estimates step-level advantages through
hierarchical grouping, while GraphGPO~\citep{cheng2026beyond}
leverages state-transition graphs to refine credit assignment.
Rather than directly re-estimating step-level RL advantages,
SCOPE complements outcome-based RL with hindsight error
localization and targeted corrective distillation.

\subsection{Hindsight Skills for LLM Agents}

Experience can guide agents through verbal reflection, reusable
skills, or explicit memory, as demonstrated by
Reflexion~\citep{shinn2023reflexion},
Voyager~\citep{wang2023voyager}, and ExpeL~\citep{zhao2024expel}.
Recent methods further internalize such guidance through policy
training~\citep{madaan2023self,yang2026opid,wang2026skill,wang2026agentopsd}.
SDAR~\citep{lu2026self} retrieves skills from an external skill bank
to provide privileged teacher context, while
Skill-SD~\citep{wang2026skill} derives teacher-only skills
from completed agent trajectories.
OPID~\citep{yang2026opid} constructs hierarchical episode- and
step-level skills through critical-first routing, falling back
to episode-level guidance at non-critical steps.

SEED~\citep{wu2026seed} introduces a policy-synchronized
hindsight analyzer. After a supervised warm-up, the policy
acts as both actor and analyzer, generating trajectory-level
skills that evolve with online training.
However, conditioning distillation on such skills across
all steps can introduce off-target signals at unrelated
decisions.
GRSD~\citep{zheng2026group} contrasts policy-generated
reflections from successful and failed trajectories to
construct group-level guidance for turn-level credit assignment.
In contrast, SCOPE internalizes critical-step analysis,
generates step-specific corrective skills, and restricts
their supervision to identified errors in both successful
and failed trajectories.

\subsection{On-Policy Self-Distillation}

Knowledge distillation transfers teacher guidance to a student
model~\citep{hinton2015distilling,kim2016sequence}.
On-policy distillation instead evaluates student-generated
sequences, reducing the distribution mismatch associated with
offline demonstrations~\citep{ross2011reduction,agarwal2024policy}.
On-policy self-distillation (OPSD) constructs a teacher from
the same model under privileged context~\citep{zhao2026self,liu2026eopsa}.
RLSD~\citep{yang2026self} uses teacher--student discrepancies
to modulate token-level update magnitudes while retaining
reward-determined optimization directions.
SDAR~\citep{lu2026self} employs a sigmoid gate that
emphasizes positive teacher endorsements while attenuating
negative rejections, improving auxiliary training stability.
AgentOPSD~\citep{wang2026agentopsd} aggregates token-level
discrepancies and recursively updates belief states for
turn-level credit assignment.
SOD~\citep{zhong2026sod} adaptively weights step-level
distillation according to teacher--student divergence.

However, teacher--student discrepancies do not necessarily
reflect task-relevant corrections.
EOPSA~\citep{liu2026eopsa} highlights how privileged prompting
can introduce task-irrelevant stylistic shifts that dilute
useful distillation signals.
In agentic tasks, skill conditioning may similarly affect
decisions unrelated to the intended correction.
Unlike divergence-based weighting, SCOPE identifies
critical errors through self-generated hindsight analysis
and restricts auxiliary distillation to these steps.
A bounded signed coefficient further limits extreme
teacher--student updates while preserving both
positive and negative corrections.

% METHOD-BEGIN
\section{Method}
\label{sec:method}

\begin{figure}[t]
    \centering
    \includegraphics[width=\columnwidth]{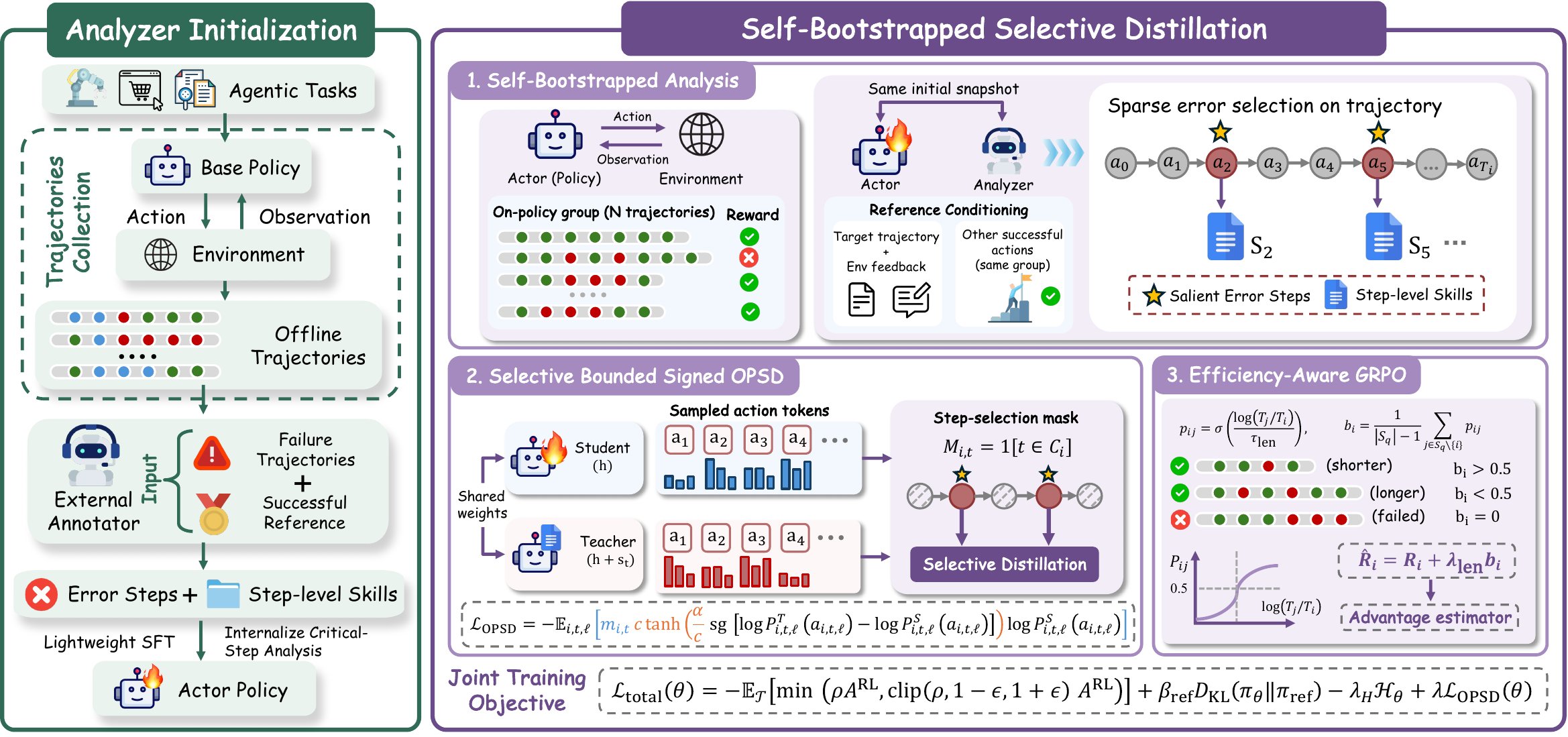}
    % \vspace{-2em}
    \caption{Overview of SCOPE. Supervised initialization teaches the
    policy to analyze trajectories and generate corrective skills.
    During online training, the policy collects and analyzes its own
    rollouts, assigns skills to selected error steps, and combines
    selective self-distillation with efficiency-aware GRPO.}
    \label{fig:method_overview}
\end{figure}

SCOPE supplements outcome-based RL with corrective supervision at
selected error steps (Figure~\ref{fig:method_overview}).
A supervised warm-up teaches the policy to locate
mistakes and generate corrective skills. During online training, the
same policy analyzes its rollouts, and a skill-conditioned self-teacher
scores the sampled actions at the selected steps. Bounded signed
self-distillation transfers these corrections to the student, while a
success-relative length reward encourages efficient task completion.

\subsection{Agent Interaction and Policy Snapshots}
\label{sec:problem_formulation}

For a task $q$, the agent receives observation $o_t$ and generates
a response $a_t=(y_{t,1},\ldots,y_{t,L_t})$ from its available interaction
history $h_t$:
\begin{equation}
    h_t=(q,o_1,a_1,\ldots,o_t),\qquad
    a_t\sim\pi_\theta(\cdot\mid h_t).
\end{equation}
The response contains the model-generated reasoning and action tokens
for one environment interaction. The environment executes the action
and returns $o_{t+1}$. A completed trajectory
$\tau_i=\{(h_{i,t},a_{i,t},o_{i,t+1})\}_{t=1}^{T_i}$ contains $T_i$
interaction steps and receives a task reward $R_i$. We denote its
binary success indicator by $z_i\in\{0,1\}$, where $z_i=1$ indicates
successful task completion. Interaction length
$T_i$ and response-token length $L_{i,t}$ are distinct quantities.

At outer iteration $u$, we take a policy snapshot
$\bar\theta_u=\theta_u$ and sample $N$ trajectories for each task:
\begin{equation}
    \mathcal G_q^{(u)}=\{\tau_i\}_{i=1}^{N},\qquad
    \tau_i\sim\pi_{\bar\theta_u}(\cdot\mid q).
\end{equation}
This snapshot supplies the actor, the hindsight analyzer, and the
privileged teacher used to construct the update data. The trainable
student starts from this checkpoint and updates its parameters
$\theta$ across successive minibatches, while $\bar\theta_u$
remains fixed for the current outer iteration.

\subsection{Self-Bootstrapped Error Analysis}
\label{sec:critical_step_analysis}

\paragraph{Supervised initialization}
We first train the policy to analyze trajectories through supervised
fine-tuning, following SEED~\citep{wu2026seed}. GLM-5.2~\citep{zai2026glm52} labels errors
in a small pool of base-policy trajectories and writes a short
corrective skill for each error or group of related errors.
A skill entry is represented as
$(\mathcal D_{i,k},s_{i,k},e_{i,k})$, where $\mathcal D_{i,k}$ is a
nonempty set of target-step indices, $s_{i,k}$ is the corrective skill,
and $e_{i,k}$ contains the error explanation and supporting evidence.
A singleton target set describes one mistake; several steps may share
a skill when they require the same correction.

Let $x_i$ be the serialized analysis input and $Y_i$ its structured
annotation. The warm-up minimizes the autoregressive objective
\begin{equation}
\label{eq:analyzer_sft}
    \mathcal L_{\mathrm{SFT}}(\theta)
    =-\mathbb E_{(x_i,Y_i)\sim\mathcal D_{\mathrm{SFT}}}
    \left[\sum_{r=1}^{|Y_i|}
    \log\pi_\theta(Y_{i,r}\mid x_i,Y_{i,<r})\right],
\end{equation}
where $Y_{i,r}$ is an annotation token. The resulting checkpoint
initializes online training. External annotation is confined to this
offline stage. Appendix~\ref{app:experimental_details} describes
data preparation, and Appendix~\ref{app:analysis_prompts} gives the
SFT and analyzer prompt templates.

\paragraph{Online hindsight analysis}
The analyzer inspects successful and failed trajectories separately.
For a successful trajectory, it uses the trajectory's own observations,
responses, and action-validity feedback to identify avoidable mistakes.
For a failed trajectory, it additionally receives the action sequences
of successful trajectories from the same on-policy task group. With
$\operatorname{Act}(\tau)$ denoting an extracted action sequence, the
reference context is
\begin{equation}
\label{eq:analysis_references}
\mathcal R_i=
\begin{cases}
\{\operatorname{Act}(\tau_j):\tau_j\in\mathcal G_q^{(u)},\ z_j=1\},
    & z_i=0,\\
\varnothing, & z_i=1.
\end{cases}
\end{equation}
Successful references demonstrate feasible task progress. The analyzer
judges each decision using the information available at that step and
the resulting environmental feedback, allowing valid alternative
strategies and reasonable exploration.

The policy snapshot generates the analysis under an analysis-specific
prompt:
\begin{equation}
\label{eq:online_analysis}
\begin{aligned}
    \widehat Y_i&\sim
    \mathcal A_{\bar\theta_u}(\cdot\mid\tau_i,\mathcal R_i),\\
    \widehat Y_i&=
    \{(\mathcal D_{i,k},\widehat s_{i,k},\widehat e_{i,k})\}_{k=1}^{J_i},
    \qquad J_i\leq K_{\max}.
\end{aligned}
\end{equation}
The limit $K_{\max}$ bounds the number of skill entries,
not the total number of selected steps. Parsing removes nonexistent
target indices and duplicate assignments, so the retained target sets
are disjoint. The selected-step set is
\begin{equation}
\label{eq:critical_step_set}
    \mathcal C_i=\bigcup_{k=1}^{J_i}\mathcal D_{i,k},\qquad
    m_{i,t}=\mathbb I[t\in\mathcal C_i].
\end{equation}
Each selected step receives the skill of its assigned entry, denoted
$\widehat s_{i,t}$. The analyzer may return an empty set when no clear
error is supported, including when a failed group supplies no
successful reference. An unselected step receives no auxiliary
distillation signal.

At the next outer iteration, the updated policy supplies both the
actor and analyzer.

\subsection{Bounded Signed Self-Distillation}
\label{sec:selective_opsd}

\paragraph{Step-specific teacher guidance}
We use skill-conditioned on-policy self-distillation to transfer
hindsight guidance into the ordinary policy~\citep{wang2026skill,yang2026opid}.
For each selected step, a context augmentation function $H$ inserts
its assigned skill into the ordinary policy input:
$\widetilde h_{i,t}=H(h_{i,t},\widehat s_{i,t})$. We retain the original
sampled response and its within-response prefix. Teacher and student
score the same token, using parameters
$\vartheta_T=\bar\theta_u$ and $\vartheta_S=\theta$ and contexts
$h^T_{i,t}=\widetilde h_{i,t}$ and $h^S_{i,t}=h_{i,t}$:
\begin{equation}
\label{eq:teacher_student}
    \ell^b_{i,t,\ell}
    =\log\pi_{\vartheta_b}
      (y_{i,t,\ell}\mid h^b_{i,t},y_{i,t,<\ell}),
    \quad b\in\{T,S\}.
\end{equation}
Teacher log-probabilities are computed once before actor optimization
and cached for the selected steps. Student log-probabilities are
recomputed at each training forward pass as $\theta$ changes.

\paragraph{Bounding the signed learning signal}
Figure~\ref{fig:opsd_stability} shows an anomaly during minibatch
optimization. At outer update 26, the pre-optimization gap is
0.284 nats. During training, the student's mean sampled-token
log-probability drops to $-2716.0$ nats, while the cached teacher
mean remains at $-0.440$ nats. The training-stage gap reaches
2715.5 nats, with a non-finite recorded gradient norm.

Teacher scores are cached before actor optimization. Successive
minibatch updates change $\theta$, so training forwards compare
updated student scores with the fixed teacher cache. A sharp drop
in student probabilities on sampled tokens can therefore turn a
moderate pre-optimization gap into an extreme training-stage
discrepancy. In our training runs, the observed instability arises
from the raw-gap OPSD auxiliary loss rather than the GRPO objective:
the discrepancy directly scales the student's score gradient and
can amplify auxiliary updates.

% OPSD-STABILITY-FIGURE-BEGIN
\begin{figure}[t]
\centering
% \vspace{-1em}
\includegraphics[width=\columnwidth]{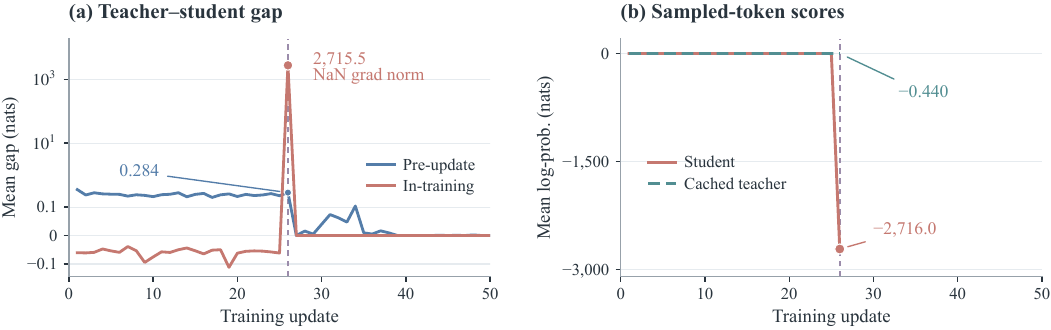}%
% \vspace{-1em}
\caption{\textbf{Anomaly during minibatch optimization.}
(a) Pre-optimization and training-stage gaps (symmetric-log axis).
(b) Current student and cached teacher sampled-token log-probabilities.
Dashed lines mark outer update 26.}
\label{fig:opsd_stability}
\end{figure}
% OPSD-STABILITY-FIGURE-END

We control gap-driven amplification by replacing the raw gap with a
bounded signed coefficient:
\begin{equation}
\label{eq:dist_advantage}
\begin{aligned}
    \Delta_{i,t,\ell}(\theta)
      &=\operatorname{sg}
        [\ell^T_{i,t,\ell}-\ell^S_{i,t,\ell}(\theta)],\\
    \widehat A^{\mathrm{dist}}_{i,t,\ell}
      &=c\,\tanh\!\left(
          \frac{\alpha\Delta_{i,t,\ell}}{c}\right),
\end{aligned}
\end{equation}
where $\operatorname{sg}$ stops gradients through the coefficient.
The parameters $\alpha>0$ and $c>0$ control the slope at zero and
saturation magnitude: $\widehat A^{\mathrm{dist}}\approx\alpha\Delta$
for small gaps, and $|\widehat A^{\mathrm{dist}}|<c$ for finite gaps.
This symmetric transformation preserves the gap's sign and gives
zero guidance when teacher and student agree on the sampled token.

We optimize a sampled-token surrogate rather than a full-distribution
KL, using tokens from the sampled responses. Let $\mathcal T$ denote
the valid, retained response-token positions in the current training
microbatch, and let $\mathcal T_{\mathrm{sel}}$
contain only those positions belonging to selected error steps.
Writing $N_{\mathrm{sel}}=|\mathcal T_{\mathrm{sel}}|$, the
distillation loss is simply the mean over these selected tokens:

\begin{equation}
\label{eq:selective_opsd}
\begin{aligned}
\mathcal L_{\mathrm{OPSD}}(\theta)
={}&-\frac{c}{N_{\mathrm{sel}}}
\sum_{(i,t,\ell)\in\mathcal T}m_{i,t}
\cdot\tanh\!\left(
\frac{\alpha}{c}\operatorname{sg}
[\ell^T_{i,t,\ell}-\ell^S_{i,t,\ell}(\theta)]
\right)
\ell^S_{i,t,\ell}(\theta).
\end{aligned}
\end{equation}

Unselected steps do not enter this loss or its token count. If
$N_{\mathrm{sel}}=0$, we set $\mathcal L_{\mathrm{OPSD}}=0$.

For each selected token, the contribution to the policy gradient is
$-\widehat A^{\mathrm{dist}}\nabla_\theta\ell^S/N_{\mathrm{sel}}$.
The derivative with respect to its student log-probability has
magnitude at most $c/N_{\mathrm{sel}}$. With $c=1$, the coefficient
is confined to $[-1,1]$, controlling the influence of extreme
teacher--student discrepancies on distillation updates. The signed
transformation reinforces positive-gap tokens and suppresses
negative-gap tokens, retaining corrective guidance in both directions.

SDAR~\citep{lu2026self} also addresses OPSD instability with a bounded
sigmoid gate. Its nonnegative auxiliary coefficient attenuates
teacher rejections without retaining their suppressive direction.
SCOPE preserves this negative guidance at selected error steps through
a signed coefficient, while bounding both update directions.
Appendix~\ref{app:bounded_distillation} derives the fixed-data
log-cosh interpretation and auxiliary-gradient bounds;
Appendix~\ref{app:selection_influence} bounds the contribution of
irrelevant selected tokens.

\subsection{Success-Relative Pairwise Efficiency}
\label{sec:pairwise_efficiency}

Selective distillation corrects identified errors. Successful
trajectories can also contain redundant interactions.
Length-aware rewards have been explored to balance correctness and
computational cost in reasoning models. L1~\citep{aggarwal2025l1}
combines correctness with adherence to a specified token-length budget,
while ShorterBetter~\citep{yi2026shorterbetter} uses the shortest
correct response in a sampled group as a dynamic length reference.
For multi-turn agents, we measure efficiency by the number of
environment interactions and compare successful trajectories for the
same task within the current rollout group. This provides a
policy-relative reference as successful interaction lengths evolve
during training.

Let $\mathcal S_q=\{i:\tau_i\in\mathcal G_q^{(u)},\ z_i=1\}$.
For two distinct successful trajectories, we use a Bradley--Terry
pairwise comparison~\citep{bradley1952rank} with utility
$-\log T_i/\tau_{\mathrm{len}}$:

\begin{equation}
\label{eq:pairwise_length}
    p_{ij}=\sigma\!\left(
       \frac{\log T_j-\log T_i}{\tau_{\mathrm{len}}}
    \right),
\end{equation}
where $\sigma(x)=(1+e^{-x})^{-1}$ and $\tau_{\mathrm{len}}>0$
controls the smoothness of the preference. The score exceeds $1/2$ when
trajectory $i$ is shorter than trajectory $j$. Averaging over the
other successful trajectories gives
\begin{equation}
\label{eq:pairwise_score}
    b_i=
    \begin{cases}
    \displaystyle\frac{1}{|\mathcal S_q|-1}
       \sum_{j\in\mathcal S_q\setminus\{i\}}p_{ij},
        & |\mathcal S_q|\geq2,\\[6pt]
    \frac12, & |\mathcal S_q|=1,
    \end{cases}\quad i\in\mathcal S_q.
\end{equation}
Each complete trajectory enters this comparison once, regardless of
how many turn records it produces. We add the resulting efficiency
bonus to the original task reward $R_i$:
\begin{equation}
\label{eq:length_reward}
    R_i^{\mathrm{pair}}=
    \begin{cases}
        R_i+\lambda_{\mathrm{len}}b_i, & z_i=1,\\
        R_i, & z_i=0.
    \end{cases}
\end{equation}
Here $\lambda_{\mathrm{len}}>0$ controls the bonus strength, and
length preferences apply only to successes. A lone success receives
$b_i=1/2$, yielding $R_i+\lambda_{\mathrm{len}}/2$ without
relative-length information. A group with no successes receives no
efficiency bonus.

The comparison is invariant to a common positive rescaling of all
lengths, because $\log(aT_j)-\log(aT_i)=\log T_j-\log T_i$.
The score therefore depends on proportional length differences,
without requiring an absolute target length. Moreover,
$p_{ij}+p_{ji}=1$, so the mean $b_i$ over successes is $1/2$.
The pairwise term allocates a bounded efficiency bonus among successful
trajectories with mean bonus $\lambda_{\mathrm{len}}/2$.
Appendix~\ref{app:pairwise_properties} proves the ordering, scale
invariance, and bonus-conservation properties of this reward.

\subsection{Joint Policy Optimization}
\label{sec:joint_optimization}

We combine GRPO with selective self-distillation. 
For each task $q$, we normalize the task-and-length
reward $R_i^{\mathrm{pair}}$ across $N$ sampled trajectories:
\begin{equation}
\label{eq:trajectory_advantage}
A_i^{\mathrm{RL}}
=\frac{R_i^{\mathrm{pair}}-\mu_q}
{\sigma_q+\epsilon_A},
\end{equation}
where $\mu_q$ and $\sigma_q$ denote the mean and standard
deviation of the rewards within the task group, and
$\epsilon_A>0$ ensures numerical stability.
All steps and tokens in trajectory $\tau_i$ share
the same advantage $A_i^{\mathrm{RL}}$.
With the clipped surrogate from PPO~\citep{schulman2017proximal},
the GRPO objective is
\begin{equation}
\label{eq:grpo_objective}
\begin{aligned}
    \mathcal L_{\mathrm{GRPO}}(\theta)
      ={}&-\mathbb E_{\mathcal T}\!\big[\min\{
          \rho A^{\mathrm{RL}},
          \operatorname{clip}(\rho,1-\epsilon,1+\epsilon)
          A^{\mathrm{RL}}\}\big]\\
        &+\beta_{\mathrm{ref}}
          D_{\mathrm{KL}}(\pi_\theta\|\pi_{\mathrm{ref}})
          -\lambda_H\mathcal H_\theta.
\end{aligned}
\end{equation}
Here $\mathbb E_{\mathcal T}$ averages over retained response tokens,
with $A^{\mathrm{RL}}=A_i^{\mathrm{RL}}$ for tokens belonging to
trajectory $i$; token indices are omitted inside the expectation.
The importance ratio is
$\rho_{i,t,\ell}(\theta)=\exp(\ell^S_{i,t,\ell}(\theta)
-\ell^{\mathrm{old}}_{i,t,\ell})$, where $\ell^{\mathrm{old}}$ is
the sampled-token log-probability under the rollout policy.
The parameter $\epsilon$ is the clipping radius.
The coefficients $\beta_{\mathrm{ref}}$ and
$\lambda_H$ weight the KL penalty to the fixed reference policy
$\pi_{\mathrm{ref}}$ and the policy entropy $\mathcal H_\theta$,
respectively. The entropy bonus encourages continued exploration.

We add the selective OPSD loss to the RL objective:
\begin{equation}
\label{eq:joint_objective}
    \mathcal L_{\mathrm{total}}(\theta)
      =\mathcal L_{\mathrm{GRPO}}(\theta)
       +\lambda\mathcal L_{\mathrm{OPSD}}(\theta).
\end{equation}
The coefficient $\lambda$ balances outcome-based
learning with corrective supervision at selected error steps.
After each outer iteration, the updated policy serves as
actor and analyzer for the next round.
% METHOD-END

\section{Experiments}
\label{sec:experiments}

We evaluate SCOPE through six research questions:

\begin{itemize}[
    leftmargin=1.5em,
    labelsep=0.5em,
    itemsep=2pt,
    topsep=3pt,
    parsep=0pt
]
    \item \textbf{RQ1: Overall Effectiveness.}
    Does SCOPE outperform outcome-based RL and
    skill-guided baselines across tasks and backbones?
    (Section~\ref{sec:main_results})

    \item \textbf{RQ2: Error Localization.}
    What limitations arise from trajectory-level
    supervision, and can SCOPE reliably identify
    critical error steps?
    (Sections~\ref{sec:supervision_diagnostics}
    and~\ref{sec:error_localization})

    \item \textbf{RQ3: Component Contributions.}
    How do supervised initialization, selective
    skill guidance, and pairwise length rewards
    affect performance and efficiency?
    (Section~\ref{sec:ablation_studies})

    \item \textbf{RQ4: Optimization Behavior.}
    How do training efficiency, stability, and
    hyperparameter choices affect SCOPE?
    (Sections~\ref{sec:training_dynamics}
    and~\ref{sec:parameter_sensitivity})

    \item \textbf{RQ5: Unseen Generalization.}
    Does SCOPE generalize to unseen tasks while
    improving interaction efficiency?
    (Section~\ref{sec:unseen_generalization})
    
\item \textbf{RQ6: Larger-Model Scalability.}
Can SCOPE maintain its effectiveness and
interaction efficiency when extended to
Qwen3-32B?
(Section~\ref{sec:scaling_32b})

\end{itemize}

\subsection{Experimental Setup}
\label{sec:experimental_setup}

\paragraph{Benchmarks and metrics}
We evaluate SCOPE on ALFWorld and WebShop, following the evaluation
protocols of SDAR~\citep{lu2026self}.
\textbf{ALFWorld}~\citep{shridhar2020alfworld} tests household task
completion through text-based interaction. Each validation round contains 128 episodes
from the seen split, covering six task families: Pick,
Look, Clean, Heat, Cool, and Pick2. We report success rates by family
and the overall success rate across the 128 episodes
(micro-average).
\textbf{WebShop}~\citep{yao2022webshop} requires agents to search,
inspect, and purchase products satisfying a user's requirements.
We evaluate on the same fixed set of 128 tasks and report two metrics:
the mean normalized task score, which measures how well the purchased
product meets the requirements, and the success rate, which counts
fully completed tasks.
All reported scores are percentages.

\paragraph{Backbones and evaluation context}
We use Qwen2.5-3B-Instruct, Qwen2.5-7B-Instruct~\citep{hui2024qwen2},
and Qwen3-1.7B~\citep{yang2025qwen3} to examine performance across
model sizes and generations. Comparisons are made within each
backbone. At evaluation, SCOPE acts from the ordinary interaction
history without hindsight skills or analyzer calls.
For the main ALFWorld and WebShop comparison, we report SCOPE's
validation metrics logged at training step 150.
The corresponding single-run validation curves
for both Qwen2.5 backbones are provided in
Appendix~\ref{app:validation_curves}.
An asterisk marks a baseline that receives skills at evaluation;
the other baselines use their ordinary policy inputs.
Key training settings and diagnostic protocols are provided in
Appendix~\ref{app:experimental_details}.

\paragraph{Baselines}
We compare against eight baselines.
\textbf{Vanilla} evaluates the original backbone without additional
agentic training.
\textbf{Skill-Prompt*} supplies retrieved skills in the evaluation
prompt without updating the model, following the prompting baseline
in SDAR~\citep{lu2026self}.
\textbf{GRPO}~\citep{shao2024deepseekmath} learns from group-relative
task rewards without an auxiliary distillation objective.

The remaining methods incorporate self-distillation or hindsight
guidance into agentic RL.
\textbf{AgentOPSD}~\citep{wang2026agentopsd} aggregates teacher--student
gaps by turn and uses recursive belief updates to assign turn-level
credit.
\textbf{SDAR}~\citep{lu2026self} retrieves teacher guidance from a
skill bank and controls auxiliary self-distillation with bounded
token-level gates.
\textbf{OPID}~\citep{yang2026opid} uses an external analyzer to
construct hierarchical skills, retaining episode-level guidance
at non-critical steps.
\textbf{GRSD}~\citep{zheng2026group} uses policy-generated
reflections scored by an external judge to construct group-level
guidance and refine turn-level credit assignment.
\textbf{SEED}~\citep{wu2026seed} initializes hindsight analysis through
SFT and iteratively uses the updated policy to generate and distill
trajectory-level skills.
% EXPERIMENT-SETUP-END

% MAIN-TABLE-BEGIN
% Requires: booktabs, xcolor[table], adjustbox, hyperref.
% Reuses the author's \scopebest, \scopesecond, and \scopetablecite macros.
% Preserve existing color definitions: scopebestcol, scopesecondcol, scopeband.
% Do not redefine the palette inside this table.

\begin{table}[t]
\centering
\caption{\textbf{Main results on ALFWorld and WebShop (\%).}
ALFWorld reports task success; WebShop reports normalized score and
success rate. * denotes skills at evaluation.
\colorbox{scopebestcol}{\textbf{Best}} and
\colorbox{scopesecondcol}{\underline{second-best}} results are highlighted.}
\label{tab:main_results}
\begingroup
\hypersetup{colorlinks=true,citecolor=scopecitecol}
\rmfamily

% A little larger and airier than the compressed version.
\fontsize{9.1}{11.4}\selectfont
\setlength{\tabcolsep}{2.8pt}
\renewcommand{\arraystretch}{1.10}
\setlength{\aboverulesep}{0.33ex}
\setlength{\belowrulesep}{0.33ex}
\setlength{\heavyrulewidth}{0.7pt}
\setlength{\lightrulewidth}{0.38pt}
\setlength{\cmidrulewidth}{0.38pt}

\begin{adjustbox}{max width=\linewidth}
\begin{tabular}{@{}ll*{9}{c}@{}}
\toprule
& & \multicolumn{7}{c}{\textbf{ALFWorld}}
    & \multicolumn{2}{c}{\textbf{WebShop}} \\
\cmidrule(lr){3-9}\cmidrule(l){10-11}
\textbf{Method} & \textbf{Skill Type}
& \textbf{Pick} & \textbf{Look} & \textbf{Clean} & \textbf{Heat}
& \textbf{Cool} & \textbf{Pick2} & \textbf{Avg}
& \textbf{Score} & \textbf{Succ.} \\
\midrule

\rowcolor{scopeband}\multicolumn{11}{@{\hspace{4pt}}l@{}}{\rule{0pt}{2.65ex}\textit{Qwen2.5-3B-Instruct}} \\
\addlinespace[1pt]
Vanilla~\scopetablecite{hui2024qwen2} & None & 44.4 & 11.1 & 6.2 & 15.4 & 28.6 & 12.5 & 21.9 & 6.7 & 0.8 \\
Skill-Prompt*~\scopetablecite{xia2026skillrl} & Skill bank & 51.7 & 66.7 & 48.4 & 0.0 & 4.3 & 10.0 & 28.9 & 0.2 & 0.8 \\
GRPO~\scopetablecite{shao2024deepseekmath} & None & 91.2 & 62.5 & \scopesecond{96.2} & 61.9 & 65.0 & 47.4 & 75.0 & 79.8 & 63.3 \\
AgentOPSD~\scopetablecite{wang2026agentopsd} & Skill bank & \scopebest{100.0} & 68.8 & 86.4 & 66.7 & \scopesecond{84.6} & 73.7 & 84.4 & \scopebest{90.4} & 69.5 \\
SDAR~\scopetablecite{lu2026self} & Skill bank & \scopesecond{97.1} & 62.5 & \scopebest{100.0} & 61.9 & 75.0 & 84.2 & 84.4 & 85.0 & 68.0 \\
OPID~\scopetablecite{yang2026opid} & Hybrid skills & 92.7 & \scopebest{100.0} & 88.9 & 70.0 & 84.2 & 70.0 & 84.3 & 85.0 & 74.2 \\
GRSD~\scopetablecite{zheng2026group} & Traj. skill & \scopebest{100.0} & \scopesecond{81.3} & 91.3 & \scopesecond{88.2} & 70.8 & \scopesecond{85.7} & 86.7 & 86.9 & 76.5 \\
SEED~\scopetablecite{wu2026seed} & Traj. skill & \scopebest{100.0} & \scopebest{100.0} & \scopebest{100.0} & \scopebest{100.0} & 70.6 & 80.0 & \scopesecond{91.8} & \scopesecond{88.5} & \scopesecond{78.9} \\
\addlinespace[2pt]
\textbf{SCOPE (Ours)} & \textbf{Step skills} & \scopebest{100.0} & \scopebest{100.0} & \scopebest{100.0} & \scopebest{100.0} & \scopebest{85.7} & \scopebest{100.0} & \scopebest{97.7} & 87.9 & \scopebest{85.2} \\

\midrule
\addlinespace[2pt]
\rowcolor{scopeband}\multicolumn{11}{@{\hspace{4pt}}l@{}}{\rule{0pt}{2.65ex}\textit{Qwen2.5-7B-Instruct}} \\
\addlinespace[1pt]
Vanilla~\scopetablecite{hui2024qwen2} & None & 36.1 & 22.2 & 3.1 & 0.0 & 0.0 & 0.0 & 12.5 & 5.9 & 1.6 \\
Skill-Prompt*~\scopetablecite{xia2026skillrl} & Skill bank & 51.7 & 50.0 & 32.3 & 5.3 & 4.3 & 0.0 & 23.4 & 1.7 & 0.8 \\
GRPO~\scopetablecite{shao2024deepseekmath} & None & 91.2 & \scopesecond{87.5} & 96.2 & 81.0 & 65.0 & 57.9 & 81.2 & 80.9 & 72.6 \\
AgentOPSD~\scopetablecite{wang2026agentopsd} & Skill bank & 91.2 & \scopesecond{87.5} & \scopebest{100.0} & 90.5 & 75.0 & 84.2 & 89.1 & \scopesecond{90.2} & 79.7 \\
SDAR~\scopetablecite{lu2026self} & Skill bank & \scopesecond{94.7} & 75.0 & \scopebest{100.0} & 86.7 & 68.2 & 78.9 & 85.9 & 89.4 & \scopesecond{82.8} \\
OPID~\scopetablecite{yang2026opid} & Hybrid skills & \scopebest{100.0} & 81.8 & \scopesecond{97.1} & \scopebest{100.0} & 80.8 & 80.0 & 90.0 & 85.3 & 79.7 \\
GRSD~\scopetablecite{zheng2026group} & Traj. skill & \scopebest{100.0} & \scopebest{100.0} & \scopebest{100.0} & \scopesecond{93.4} & 66.7 & 82.4 & 92.2 & \scopebest{91.6} & \scopesecond{82.8} \\
SEED~\scopetablecite{wu2026seed} & Traj. skill & \scopebest{100.0} & \scopebest{100.0} & 96.3 & 80.0 & \scopebest{100.0} & \scopebest{100.0} & \scopesecond{96.1} & 89.7 & 78.1 \\
\addlinespace[2pt]
\textbf{SCOPE (Ours)} & \textbf{Step skills} & \scopebest{100.0} & \scopebest{100.0} & \scopebest{100.0} & \scopebest{100.0} & \scopesecond{95.2} & \scopesecond{94.7} & \scopebest{98.4} & 89.6 & \scopebest{84.4} \\

\midrule
\addlinespace[2pt]
\rowcolor{scopeband}\multicolumn{11}{@{\hspace{4pt}}l@{}}{\rule{0pt}{2.65ex}\textit{Qwen3-1.7B}} \\
\addlinespace[1pt]
Vanilla~\scopetablecite{yang2025qwen3} & None & 25.0 & 22.2 & 3.1 & 0.0 & 21.4 & 4.2 & 12.5 & 46.5 & 4.7 \\
Skill-Prompt*~\scopetablecite{xia2026skillrl} & Skill bank & 10.3 & 50.0 & 16.1 & 0.0 & 0.0 & 5.0 & 9.4 & 23.0 & 2.3 \\
GRPO~\scopetablecite{shao2024deepseekmath} & None & 71.1 & 41.7 & 36.4 & 40.0 & 31.8 & 31.6 & 46.1 & 67.3 & 38.3 \\
AgentOPSD~\scopetablecite{wang2026agentopsd} & Skill bank & 35.5 & 50.0 & 31.6 & 30.0 & 46.2 & 25.0 & 35.9 & 66.8 & 47.7 \\
SDAR~\scopetablecite{lu2026self} & Skill bank & 73.5 & 25.0 & 76.9 & 33.3 & 40.0 & 36.8 & 53.9 & 76.8 & 58.6 \\
OPID~\scopetablecite{yang2026opid} & Hybrid skills & 65.9 & \scopesecond{72.7} & 66.7 & 40.0 & 63.2 & 45.0 & 58.9 & 79.6 & 64.8 \\
GRSD~\scopetablecite{zheng2026group} & Traj. skill & 92.5 & \scopebest{100.0} & \scopebest{100.0} & \scopesecond{91.3} & \scopesecond{81.3} & 88.2 & \scopebest{92.2} & \scopesecond{84.5} & 66.4 \\
SEED~\scopetablecite{wu2026seed} & Traj. skill & \scopesecond{97.6} & \scopebest{100.0} & \scopebest{100.0} & 80.0 & \scopebest{84.2} & \scopebest{90.0} & \scopesecond{92.0} & \scopebest{87.1} & \scopesecond{77.3} \\
\addlinespace[2pt]
\textbf{SCOPE (Ours)} & \textbf{Step skills} & \scopebest{100.0} & \scopebest{100.0} & \scopesecond{96.9} & \scopebest{100.0} & 66.7 & \scopesecond{89.5} & \scopebest{92.2} & 83.2 & \scopebest{78.9} \\
\bottomrule
\end{tabular}
\end{adjustbox}
\endgroup
\end{table}
% MAIN-TABLE-END

\subsection{Main Results}
\label{sec:main_results}

Table~\ref{tab:main_results} compares SCOPE with eight baselines
across three backbones, grouped by their auxiliary guidance types.
SCOPE achieves the highest WebShop success across all backbones
and the highest or tied highest ALFWorld success.

\paragraph{Beyond outcome-level supervision}
Compared with GRPO, SCOPE improves ALFWorld success by
22.7, 17.2, and 46.1 percentage points on Qwen2.5-3B,
Qwen2.5-7B, and Qwen3-1.7B, respectively.
The corresponding WebShop gains are 21.9, 11.8, and
40.6 points. These consistent improvements suggest
the benefit of complementing coarse outcome rewards
with localized corrective supervision, rather than
relying solely on trajectory-level feedback.

\paragraph{Self-bootstrapped error analysis versus external supervision}
OPID relies on an external analyzer to identify critical steps
and generate hierarchical skills, while GRSD uses an external
judge to evaluate policy-generated reflections.
In contrast, SCOPE internalizes error localization and skill
generation through supervised initialization, enabling the
policy to provide its own corrective guidance during online RL.
Compared with OPID, SCOPE improves ALFWorld success by
13.4, 8.4, and 33.3 percentage points across the three
backbones. It also outperforms GRSD by 11.0 and 6.2 points
on Qwen2.5-3B and Qwen2.5-7B, respectively, while
matching its ALFWorld success on Qwen3-1.7B.
These results support the effectiveness of self-bootstrapped
corrective guidance without external online analysis.

\paragraph{Selective step guidance versus trajectory-level skills}
Compared with trajectory-level SEED, SCOPE improves
ALFWorld success by 5.9 and 2.3 points on the
Qwen2.5-3B and Qwen2.5-7B backbones, respectively,
while gaining 6.3 points on WebShop for both.
On Qwen3-1.7B, SCOPE ties GRSD at 92.2\% ALFWorld
success and exceeds SEED and GRSD on WebShop success.
Notably, SCOPE achieves the highest WebShop success
across all backbones without consistently leading
in normalized task score, suggesting stronger exact
task completion rather than uniformly better partial
matches. These results are consistent with the value
of aligning distillation with selected error steps
instead of applying trajectory-level guidance broadly.

We further examine whether these improvements
extend to a larger 32B backbone in
Section~\ref{sec:scaling_32b}.

\subsection{Diagnosing Trajectory-Level Supervision}
\label{sec:supervision_diagnostics}

To examine the limitations of trajectory-level
supervision, we analyze the alignment between task
outcomes, step quality, and skill-conditioned
distillation signals.
Figure~\ref{fig:intro_overview}(a) reveals substantial
outcome--step mismatch in annotated WebShop trajectories:
29.8\% of steps in successful trajectories are erroneous,
while 42.4\% of steps in failed trajectories are correct.
This disagreement indicates that terminal outcomes
alone cannot reliably identify which decisions
require correction.

Figure~\ref{fig:intro_overview}(b) further reveals
that skill-unrelated steps account for 57.9\% of
nonzero sampled-token teacher--student gaps.
Notably, their negative and positive gaps have
similar magnitudes to those at skill-relevant steps
($-0.378/+0.230$ versus $-0.400/+0.249$ nats).
This suggests that gap magnitude alone provides
limited information about the relevance of
corrective guidance. Together, these findings
motivate explicit error localization and
selective error-step distillation, rather than
applying skill-conditioned supervision broadly.
Detailed diagnostic protocols are provided in
Appendix~\ref{app:diagnostic_details}.

\begin{figure}[t]
\centering
\includegraphics[width=\textwidth]{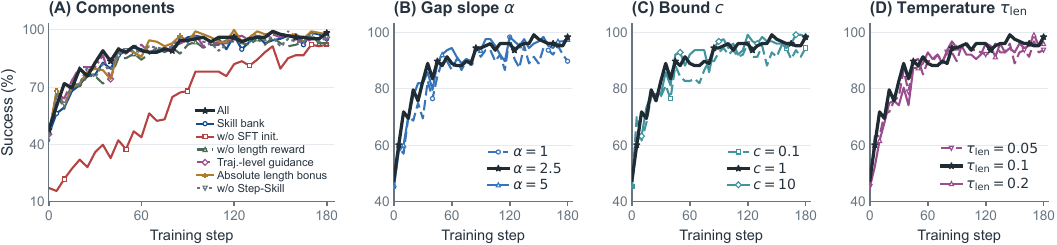}%
% \vspace{-1em}
\caption{Training dynamics on ALFWorld with Qwen2.5-3B.
(A) Component ablations. (B) Signed-gap slope $\alpha$.
(C) Tanh saturation bound $c$. (D) Pairwise length temperature
$\tau_{\mathrm{len}}$. The full setting is $\alpha=2.5$, $c=1$,
and $\tau_{\mathrm{len}}=0.1$, labeled All in (A) and by its parameter
value in (B)--(D). Trajectory-level guidance uses one skill per
trajectory and distills all steps. Lines connect observed checkpoints;
the success axes start at 10\% in (A) and 40\% in (B)--(D).}
\label{fig:scope_ablation_training}
% \vspace{-1em}
\end{figure}

% ABLATION-STUDY-BEGIN
\subsection{Ablation Studies}
\label{sec:ablation_studies}

We evaluate six variants with Qwen2.5-3B on ALFWorld
(Table~\ref{tab:scope_ablations}) to examine the contributions of
SCOPE's components and different forms of guidance. Together with
the parameter sweeps in Section~\ref{sec:parameter_sensitivity},
these experiments also examine the effect of guidance strength.
All endpoint comparisons in this section use training step 180,
where the full configuration achieves 98.4\% overall success.

\begin{table}[t]
\centering

\caption{\textbf{Ablation results on ALFWorld (\%).}
Qwen2.5-3B at training step 180.
We evaluate the effects of supervised initialization,
skill guidance, and length-reward design.
Traj.-level guidance assigns one skill per trajectory
and distills across all steps; skill-bank guidance
uses retrieved skills. Overall success is micro-averaged
across evaluation episodes. Best values, including ties,
are bold.}

\label{tab:scope_ablations}

\begingroup
\rmfamily
\fontsize{9.2}{11.2}\selectfont
\setlength{\tabcolsep}{3.6pt}
\renewcommand{\arraystretch}{1.12}

% Booktabs spacing
\setlength{\aboverulesep}{0.3ex}
\setlength{\belowrulesep}{0.3ex}
\setlength{\heavyrulewidth}{0.65pt}
\setlength{\lightrulewidth}{0.35pt}
\setlength{\cmidrulewidth}{0.35pt}

\begin{adjustbox}{max width=\linewidth}
\begin{tabular}{l*{7}{c}}
\toprule
& \multicolumn{7}{c}{\textbf{ALFWorld}} \\
\cmidrule(lr){2-8}

\textbf{Variant}
& \textbf{Pick}
& \textbf{Look}
& \textbf{Clean}
& \textbf{Heat}
& \textbf{Cool}
& \textbf{Pick2}
& \textbf{Overall} \\
\midrule

\rowcolor{scopeablationfull}
\textbf{Full SCOPE}
& \textbf{100.0}
& \textbf{100.0}
& \textbf{100.0}
& \textbf{100.0}
& \textbf{100.0}
& 92.0
& \textbf{98.4} \\

\midrule

% Initialization and guidance
\rowcolor{scopeband}
\multicolumn{8}{l}{
\rule{0pt}{2.5ex}\textit{Initialization and guidance}
} \\

\quad w/o SFT initialization
& \textbf{100.0} & 92.9 & \textbf{100.0}
& 70.6 & 83.3 & \textbf{94.7} & 92.2 \\

\quad w/o Step-Skill
& \textbf{100.0} & 92.9 & \textbf{100.0}
& 94.1 & 94.4 & 89.5 & 96.1 \\

\quad Traj.-level guidance
& \textbf{100.0} & 92.9 & \textbf{100.0}
& 88.2 & \textbf{100.0} & 89.5 & 96.1 \\

\quad Skill-bank guidance
& \textbf{100.0} & \textbf{100.0} & 96.6
& 76.5 & \textbf{100.0} & \textbf{94.7} & 95.3 \\

\midrule

% Length-reward design
\rowcolor{scopeband}
\multicolumn{8}{l}{
\rule{0pt}{2.5ex}\textit{Length-reward design}
} \\

\quad w/o length reward
& 96.8 & 91.7 & \textbf{100.0}
& \textbf{100.0} & 92.3 & 75.0 & 93.0 \\

\quad Absolute length bonus
& \textbf{100.0} & \textbf{100.0}
& \textbf{100.0} & 94.1
& \textbf{100.0} & 89.5 & 97.7 \\

\bottomrule
\end{tabular}
\end{adjustbox}

\endgroup
\end{table}

\paragraph{Supervised initialization supports subsequent self-analysis}
Removing SFT initialization lowers overall success from 98.4\% to
92.2\%, a 6.3-point drop. Its effect is also visible during training:
at step 100, the full configuration reaches 96.1\%, compared with
78.1\% without SFT (Figure~\ref{fig:scope_ablation_training}(A)).
Heat is especially affected at the final checkpoint, falling from
100.0\% to 70.6\%. These results highlight the need to first
equip the model to locate errors and generate corrective guidance
before relying on its own analyses for online supervision.

\paragraph{Matching skills to target steps improves guidance}
Replacing step-specific skills and selective coverage with one skill
per trajectory and distillation over all steps reduces success to
96.1\%, a 2.3-point drop. Replacing the generated guidance with
skill-bank guidance yields 95.3\%, a 3.1-point drop. Together with
the unrelated-step signals in Section~\ref{sec:supervision_diagnostics},
these results support constructing guidance from the current
trajectory and restricting its application to the relevant decisions.
The trajectory-level variant changes skill granularity and coverage
jointly, so its result supports their design.

% \Needspace{12\baselineskip}
\begin{wrapfigure}{r}{0.42\columnwidth}
\vspace{-1em}
    \centering
    \vspace{-0.5\baselineskip}
    \includegraphics[
        width=\linewidth,
        height=0.22\textheight,
        keepaspectratio
    ]{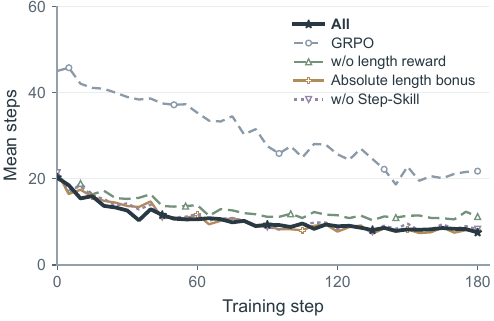}
\caption{\textbf{Interaction efficiency during training.}
Mean environment interaction steps on ALFWorld
across different training configurations.}
\label{fig:scope_response_steps}

    \label{fig:scope_response_steps}
    \vspace{-0.5\baselineskip}
\end{wrapfigure}
% \noindent

\paragraph{Pairwise length rewards promote sustained efficiency}
Removing the length reward reduces overall success
from 98.4\% to 93.0\%, while replacing the pairwise
reward with an absolute length bonus yields 97.7\%.
Unlike absolute length shaping, our pairwise reward
compares successful trajectories within each on-policy
group, adapting its preferences as the policy improves.
Figure~\ref{fig:scope_response_steps} shows that SCOPE
rapidly reduces interaction steps during early training
and maintains low interaction counts in later stages,
consistently outperforming GRPO in efficiency.
Compared with removing length rewards, pairwise shaping
generally maintains shorter trajectories after the
initial training phase. Together, these results
support efficient task completion without sacrificing
success.

% ABLATION-CURVES-BEGIN
\subsection{Training Dynamics}
\label{sec:training_dynamics}

To examine learning efficiency and optimization stability,
we compare the training trajectories of SCOPE and GRPO
and investigate the impact of unbounded OPSD updates.

\paragraph{SCOPE reaches strong performance with fewer updates}
Figure~\ref{fig:intro_overview}(c) shows that SCOPE
reaches 65.1\% mean success on WebShop at step 60,
surpassing GRPO's final 63.0\% at step 150,
requiring 60.0\% fewer online updates.
On ALFWorld, SCOPE first exceeds GRPO's final
success rate of 76.6\% at step 25, compared with
step 145 for GRPO, reducing the required updates
by 82.8\%. At step 180, SCOPE achieves 98.4\%
success, compared with 76.6\% at GRPO's final
checkpoint. These comparisons measure online
updates after supervised initialization,
demonstrating faster performance improvement
under a smaller RL update budget.

\paragraph{Bounded signed distillation stabilizes training}
Figure~\ref{fig:intro_overview}(c) further shows that
Unbounded SCOPE suffers an abrupt performance collapse
after initial improvement, whereas bounded SCOPE
maintains stable training.
Figure~\ref{fig:opsd_stability} reveals the underlying
numerical instability: sharply reduced student
log-probabilities produce extreme teacher--student
gaps, amplifying raw-gap OPSD updates during
minibatch optimization.
The bounded signed coefficient in
Equation~\ref{eq:dist_advantage} limits this
gap-driven amplification while preserving
both positive and negative corrections.
Together, these results support stable and
efficient online optimization.

% ABLATION-CURVES-END
% ABLATION-STUDY-END

% PARAMETER-SENSITIVITY-BEGIN
\subsection{Parameter Sensitivity}
\label{sec:parameter_sensitivity}

\begin{figure}[t]
    \centering
    % Left: Parameter sensitivity curves
    \begin{minipage}[c]{0.6\columnwidth}
        \centering
        \includegraphics[width=\linewidth]
        {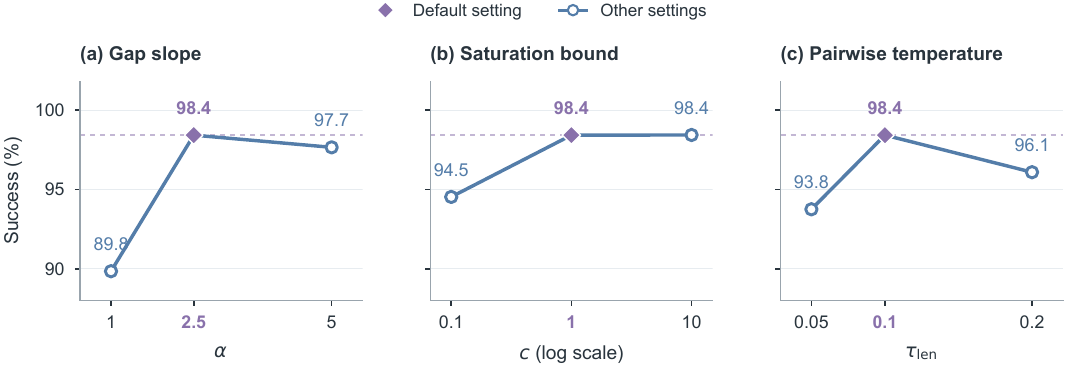}
    \end{minipage}%
    \hfill
    % Right: Task-wise sensitivity heatmap
    \begin{minipage}[c]{0.38\columnwidth}
        \centering
        \includegraphics[width=\linewidth]
        {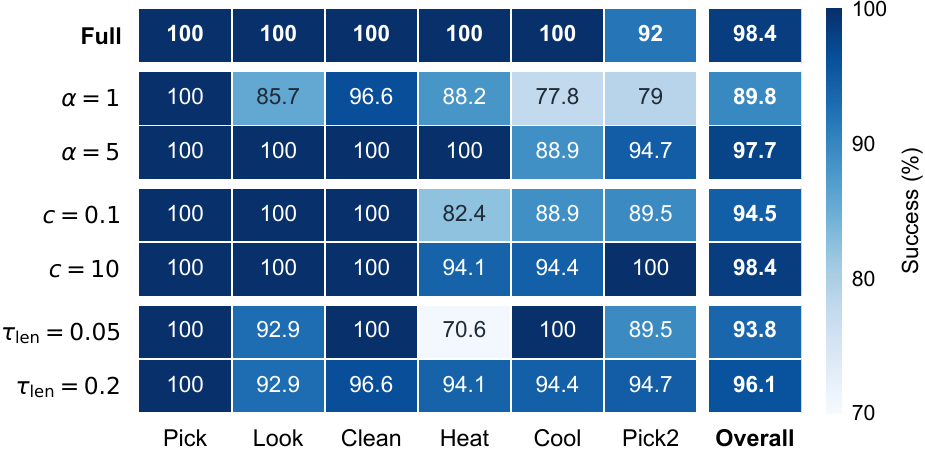}
    \end{minipage}

    \caption{\textbf{Parameter sensitivity of SCOPE.}
    Left: effects of gap slope $\alpha$, saturation
    bound $c$, and pairwise temperature $\tau_{\mathrm{len}}$
    on task success. Right: task-wise performance
    under different parameter settings, where
    \textit{Full} denotes the default configuration
    ($\alpha=2.5$, $c=1$, $\tau_{\mathrm{len}}=0.1$).}
    \label{fig:parameter_sensitivity}
\end{figure}

To examine the sensitivity of corrective supervision
and length-aware optimization, we vary the signed-gap
slope $\alpha$, saturation bound $c$, and pairwise
temperature $\tau_{\mathrm{len}}$ individually.
The remaining parameters follow the default configuration
($\alpha=2.5$, $c=1$, $\tau_{\mathrm{len}}=0.1$).
Figure~\ref{fig:parameter_sensitivity} reports overall
and task-wise ALFWorld success at training step 180.

\paragraph{The signed-gap slope controls corrective strength}
Reducing $\alpha$ to 1 lowers overall success from
98.4\% to 89.8\%, with notable declines on Cool
and Look. A smaller slope weakens the coefficient
for moderate teacher--student gaps, potentially
limiting corrective supervision.
Increasing $\alpha$ to 5 maintains 97.7\% success,
suggesting that effective learning does not depend
on the exact default slope.

\paragraph{Overly tight saturation limits correction}
Reducing $c$ to 0.1 lowers overall success to
94.5\%, with Heat dropping to 82.4\%.
This suggests that early coefficient saturation
can restrict useful corrective updates.
Both $c=1$ and $c=10$ achieve 98.4\% overall
success, indicating that SCOPE performs well
across different saturation bounds.

\paragraph{Moderate pairwise preferences favor performance}
Setting $\tau_{\mathrm{len}}$ to 0.05 or 0.2 yields
93.8\% and 96.1\% success, respectively,
compared with 98.4\% at the default 0.1.
Lower temperatures produce sharper length
preferences, while higher temperatures soften
pairwise differences. The intermediate setting
achieves the highest success in this sweep,
supporting a balance between length discrimination
and corrective learning.
% PARAMETER-SENSITIVITY-END

\begin{figure}[t]
    \centering

    % Left: Error Localization
    \begin{minipage}[t]{0.59\columnwidth}
        \centering
        \vspace{0pt}
        \includegraphics[width=\linewidth]
        {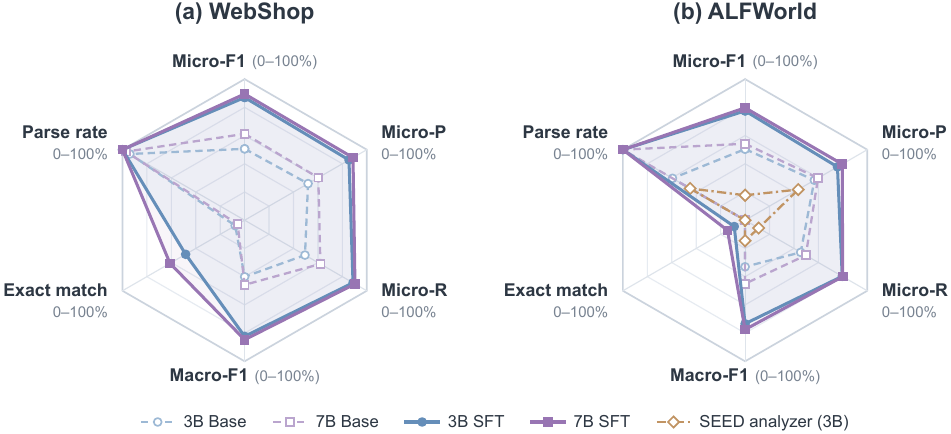}
        
\caption{\textbf{Error-step localization.}
Comparison of base, SFT-initialized, and SEED analyzers
on WebShop and ALFWorld, evaluated using
precision, recall, F1, and parsing accuracy.}
\label{fig:error_localization}
    \end{minipage}%
    \hfill
    % Right: Unseen Generalization
    % \vspace{-0.3em}
    \begin{minipage}[t]{0.39\columnwidth}
        \centering
        \vspace{0pt}
        \includegraphics[width=\linewidth]
        {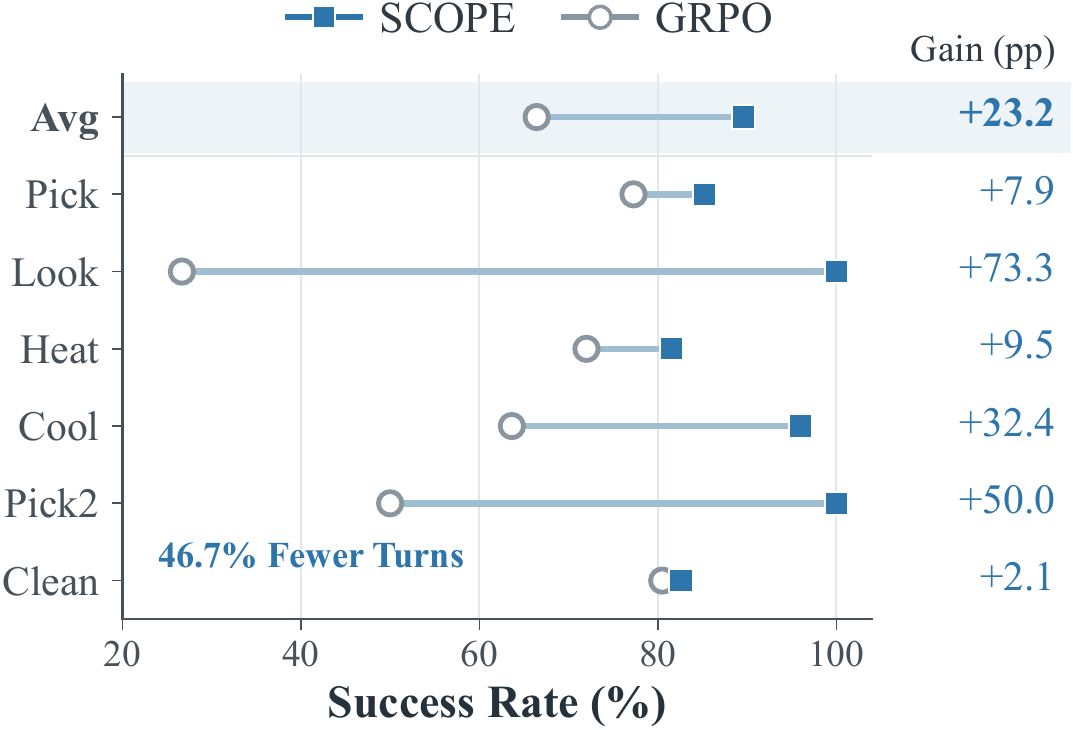}
\caption{\textbf{Unseen ALFWorld.}
Generalization to unseen instances across six
task families, comparing success and efficiency
with GRPO.}
\label{fig:scope_unseen_generalization}
    \end{minipage}
\end{figure}

\subsection{Error-Localization Capability}
\label{sec:error_localization}

Since SCOPE relies on self-generated error analysis
to select critical steps, reliable error localization
is essential for effective corrective supervision.
We therefore evaluate the localization capability
of base and SFT-initialized models on WebShop and
ALFWorld, alongside a SEED-style analyzer.
Figure~\ref{fig:error_localization} reveals three
key findings.

First, larger models generally exhibit stronger
error-step localization, with 7B outperforming 3B
across most metrics. Second, SFT substantially
improves both backbones, with 3B SFT even surpassing
7B Base in micro-F1, demonstrating that smaller
models can acquire effective error-analysis capabilities.
Third, the SEED-style analyzer exhibits markedly
weaker localization than our step-specific analyzer,
particularly on ALFWorld, suggesting that
trajectory-level hindsight training may be less
effective for fine-grained error identification.
These results support supervised initialization
for reliable online self-analysis.
Metric definitions are provided in
Appendix~\ref{app:error_localization_metrics}.

% GENERALIZATION-BEGIN
\subsection{Generalization to Unseen Tasks}
\label{sec:unseen_generalization}

\paragraph{SCOPE generalizes to unseen task instances}
To examine whether SCOPE transfers beyond the tasks
encountered during training, we evaluate its performance
on the unseen ALFWorld split, covering six household
task families. As shown in
Figure~\ref{fig:scope_unseen_generalization},
SCOPE achieves 89.6\% overall success compared with
66.4\% for GRPO, a gain of 23.2 percentage points.
Improvements span all six families, particularly
Look (+73.3 points), Pick2 (+50.0), and Cool (+32.4).
These results demonstrate that SCOPE's learned
corrective behavior generalizes to unfamiliar
task instances across diverse household routines.

\paragraph{Generalization gains accompany greater efficiency}
Beyond higher task success, SCOPE also reduces
average interactions from 18.86 to 10.06 turns,
a 46.7\% reduction relative to GRPO.
This improvement indicates that SCOPE generalizes
without relying on longer interaction sequences,
achieving more successful completions with fewer
environment interactions. Together with the
length-reward ablation, these findings support
the effectiveness of efficiency-aware agentic RL.
% GENERALIZATION-END

% 32B-SCALING-BEGIN
\subsection{Scaling SCOPE to Qwen3-32B}
\label{sec:scaling_32b}

To investigate whether SCOPE remains effective
at a larger model scale, we extend our evaluation
to Qwen3-32B on ALFWorld, covering both seen
and unseen task instances. We compare SCOPE
against GRPO using the same backbone.
Table~\ref{tab:scope_32b} reports the results,
while Figure~\ref{fig:scope_32b_curves} presents
single-run training dynamics on the seen split.

\begin{table}[t]
\centering
\caption{\textbf{Qwen3-32B results on ALFWorld.}
Seen results are reported at training step 150,
while unseen evaluation covers 134 episodes.
Task-family and overall success rates are percentages;
mean steps measure environment interactions per episode.
Best values within each split are bold.}
\label{tab:scope_32b}

\begingroup
\rmfamily
\fontsize{9.2}{11.2}\selectfont
\setlength{\tabcolsep}{3.6pt}
\renewcommand{\arraystretch}{1.12}

% Booktabs spacing
\setlength{\aboverulesep}{0.3ex}
\setlength{\belowrulesep}{0.3ex}
\setlength{\heavyrulewidth}{0.65pt}
\setlength{\lightrulewidth}{0.35pt}
\setlength{\cmidrulewidth}{0.35pt}

\begin{adjustbox}{max width=\linewidth}
\begin{tabular}{l*{8}{c}}
\toprule
& \multicolumn{7}{c}{\textbf{ALFWorld $\uparrow$}}
& {} \\
\cmidrule(lr){2-8}

\textbf{Method}
& \textbf{Pick}
& \textbf{Look}
& \textbf{Clean}
& \textbf{Heat}
& \textbf{Cool}
& \textbf{Pick2}
& \textbf{Overall}
& \shortstack{\textbf{Mean steps} $\downarrow$} \\
\midrule

% ================= Seen =================
\rowcolor{scopeband}
\multicolumn{9}{l}{
\rule{0pt}{2.5ex}\textit{Seen evaluation}
} \\

GRPO
& 96.8
& 75.0
& 89.5
& 90.0
& 80.8
& 70.0
& 85.2
& 22.0 \\

\rowcolor{scopeablationfull}
\textbf{SCOPE (Ours)}
& \textbf{100.0}
& \textbf{100.0}
& \textbf{100.0}
& \textbf{95.0}
& \textbf{100.0}
& \textbf{100.0}
& \textbf{99.2}
& \textbf{7.7} \\

\midrule

% ================= Unseen =================
\rowcolor{scopeband}
\multicolumn{9}{l}{
\rule{0pt}{2.5ex}\textit{Unseen evaluation}
} \\

GRPO
& \textbf{87.5}
& 83.3
& 77.4
& 78.3
& 85.7
& 76.5
& 81.3
& 17.4 \\

\rowcolor{scopeablationfull}
\textbf{SCOPE (Ours)}
& 83.3
& \textbf{100.0}
& \textbf{93.5}
& \textbf{91.3}
& \textbf{100.0}
& \textbf{100.0}
& \textbf{94.0}
& \textbf{9.2} \\

\bottomrule
\end{tabular}
\end{adjustbox}
\endgroup
\end{table}

\begin{figure}[t]
    \centering
    \includegraphics[width=\columnwidth]
    {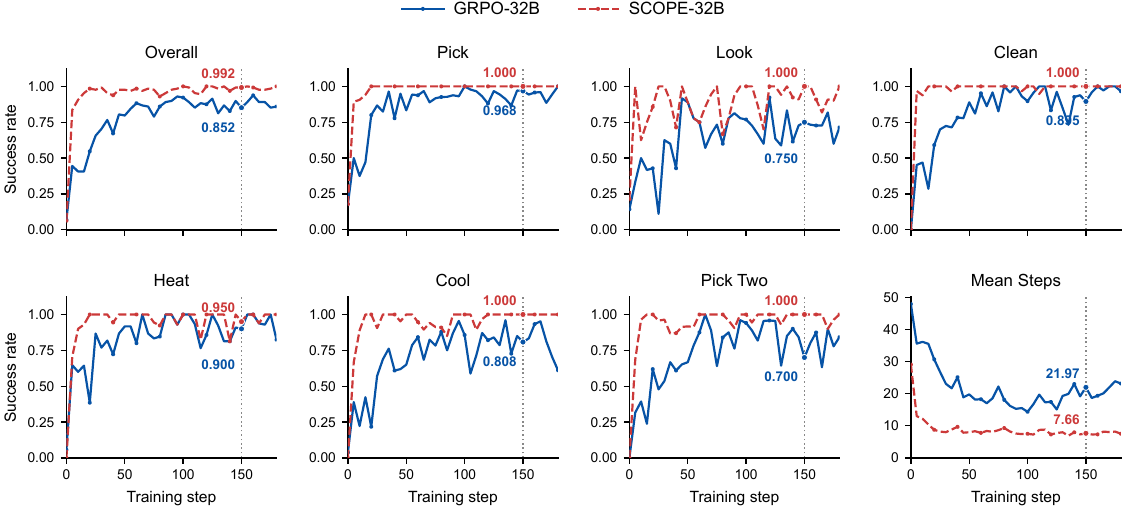}
    \caption{\textbf{Training dynamics with Qwen3-32B.}
    Single-run ALFWorld seen-evaluation curves
    comparing SCOPE and GRPO across six task
    families, overall success, and mean interaction
    steps. The dotted line marks step 150,
    corresponding to the seen results in
    Table~\ref{tab:scope_32b}.}
    \label{fig:scope_32b_curves}
\end{figure}

\paragraph{SCOPE remains effective at larger model scale}
At step 150, SCOPE achieves 99.2\% overall success
on Qwen3-32B, outperforming GRPO (85.2\%) by
14.0 percentage points. Improvements extend across
all six task families, with particularly large gains
on Pick2 (+30.0 points), Look (+25.0), and
Cool (+19.2). These results indicate that
SCOPE's targeted corrective supervision remains
effective when extended to a substantially
larger Qwen3 backbone.

\paragraph{Higher success with fewer interactions}
Figure~\ref{fig:scope_32b_curves} shows that
SCOPE improves rapidly during early training
and maintains near-ceiling success in later stages.
On seen tasks, average interactions decrease
from 22.0 to 7.7, a 65.0\% reduction.
The gains also extend to 134 unseen ALFWorld
episodes, where SCOPE achieves 94.0\% success
versus 81.3\% for GRPO (+12.7 points), while
reducing average interactions by 47.2\%.
Together, these results demonstrate that SCOPE
retains both performance and efficiency advantages
at larger model scale, including on unseen tasks.

% 32B-SCALING-END

\section{Conclusion}

In this work, we identify two limitations of trajectory-level
supervision for LLM agents: trajectory outcomes obscure
intermediate decision quality, while trajectory-level skills
can introduce off-target distillation signals.
To address these challenges, we propose \textbf{SCOPE},
which internalizes error localization and applies
self-generated, step-specific corrective guidance only
at selected error steps. Bounded signed distillation
controls gap-driven updates, while success-relative
pairwise rewards encourage efficient task completion.
Experiments on ALFWorld and WebShop across three backbones
demonstrate consistent improvements over outcome-based RL
and strong performance against skill-guided methods.
On unseen ALFWorld tasks, SCOPE achieves 89.6\% success,
outperforming GRPO by 23.2 percentage points while
reducing reported interactions by 46.7\%.
These results highlight the importance of aligning
corrective supervision with the decisions that require it,
without additional guidance at inference time.

\FloatBarrier
\begingroup
\setlength{\parskip}{0pt}
\setlength{\bibsep}{0pt plus 0.5pt}
\bibliographystyle{abbrvnat}
\bibliography{arxiv}

@article{yao2022react,
  title={React: Synergizing reasoning and acting in language models},
  author={Yao, Shunyu and Zhao, Jeffrey and Yu, Dian and Du, Nan and Shafran, Izhak and Narasimhan, Karthik and Cao, Yuan},
  journal={arXiv preprint arXiv:2210.03629},
  year={2022}
}

@article{schick2023toolformer,
  title={Toolformer: Language models can teach themselves to use tools},
  author={Schick, Timo and Dwivedi-Yu, Jane and Dess{\`\i}, Roberto and Raileanu, Roberta and Lomeli, Maria and Hambro, Eric and Zettlemoyer, Luke and Cancedda, Nicola and Scialom, Thomas},
  journal={Advances in neural information processing systems},
  volume={36},
  pages={68539--68551},
  year={2023}
}

@inproceedings{liu2024agentbench,
  title={Agentbench: Evaluating llms as agents},
  author={Liu, Xiao and Yu, Hao and Zhang, Hanchen and Xu, Yifan and Lei, Xuanyu and Lai, Hanyu and Gu, Yu and Ding, Hangliang and Men, Kaiwen and Yang, Kejuan and others},
  booktitle={International Conference on Learning Representations},
  volume={2024},
  pages={52989--53046},
  year={2024}
}

@article{qian2026toolrl,
  title={Toolrl: Reward is all tool learning needs},
  author={Qian, Cheng and Acikgoz, Emre Can and He, Qi and Wang, Hongru and Chen, Xiusi and Hakkani-Tur, Dilek and Tur, Gokhan and Ji, Heng},
  journal={Advances in Neural Information Processing Systems},
  volume={38},
  pages={105523--105553},
  year={2026}
}

@article{wang2025ragen,
  title={Ragen: Understanding self-evolution in llm agents via multi-turn reinforcement learning},
  author={Wang, Zihan and Wang, Kangrui and Wang, Qineng and Zhang, Pingyue and Li, Linjie and Yang, Zhengyuan and Jin, Xing and Yu, Kefan and Nguyen, Minh Nhat and Liu, Licheng and others},
  journal={arXiv preprint arXiv:2504.20073},
  year={2025}
}

@article{jin2025search,
  title={Search-r1: Training llms to reason and leverage search engines with reinforcement learning},
  author={Jin, Bowen and Zeng, Hansi and Yue, Zhenrui and Yoon, Jinsung and Arik, Sercan and Wang, Dong and Zamani, Hamed and Han, Jiawei},
  journal={arXiv preprint arXiv:2503.09516},
  year={2025}
}

@article{shao2024deepseekmath,
  title={Deepseekmath: Pushing the limits of mathematical reasoning in open language models},
  author={Shao, Zhihong and Wang, Peiyi and Zhu, Qihao and Xu, Runxin and Song, Junxiao and Bi, Xiao and Zhang, Haowei and Zhang, Mingchuan and Li, YK and Wu, Yang and others},
  journal={arXiv preprint arXiv:2402.03300},
  year={2024}
}

@article{luo2025agent,
  title={Agent lightning: Train any ai agents with reinforcement learning},
  author={Luo, Xufang and Zhang, Yuge and He, Zhiyuan and Wang, Zilong and Zhao, Siyun and Li, Dongsheng and Qiu, Luna K and Yang, Yuqing},
  journal={arXiv preprint arXiv:2508.03680},
  year={2025}
}

@inproceedings{wu2026spark,
  title={Spark: Strategic policy-aware exploration via dynamic branching for long-horizon agentic learning},
  author={Wu, Jinyang and Yang, Shuo and Shen, Yuhao and Zhang, Shuai and Wen, Zhengqi and Tao, Jianhua},
  booktitle={Proceedings of the 64th Annual Meeting of the Association for Computational Linguistics (Volume 1: Long Papers)},
  pages={23981--24004},
  year={2026}
}

@article{feng2026group,
  title={Group-in-group policy optimization for llm agent training},
  author={Feng, Lang and Xue, Zhenghai and Liu, Tingcong and An, Bo},
  journal={Advances in Neural Information Processing Systems},
  volume={38},
  pages={46375--46408},
  year={2026}
}

@article{cheng2026beyond,
  title={Beyond Trajectory-Level Attribution: Graph-Based Credit Assignment for Agentic Reinforcement Learning},
  author={Cheng, Xin and He, Shuo and Feng, Lang and Xu, HaiYang and Yan, Ming and Feng, Lei and An, Bo},
  journal={arXiv preprint arXiv:2605.26684},
  year={2026}
}

@article{liu2026deshortcut,
  title={DeShortcut-Align: Decoupling Spurious Shortcuts for Robust Safety Alignment in Large Reasoning Models},
  author={Liu, Qirui and Sun, Yichen and Wang, Yan and Chu, Zhixuan and Jiang, Linbo and Lin, Jianan and Ren, Kui},
  journal={arXiv preprint arXiv:2609.34896},
  year={2026}
}

@article{zhao2026self,
  title={Self-distilled reasoner: On-policy self-distillation for large language models},
  author={Zhao, Siyan and Xie, Zhihui and Liu, Mengchen and Huang, Jing and Pang, Guan and Chen, Feiyu and Grover, Aditya},
  journal={arXiv preprint arXiv:2601.18734},
  year={2026}
}

@inproceedings{agarwal2024policy,
  title={On-policy distillation of language models: Learning from self-generated mistakes},
  author={Agarwal, Rishabh and Vieillard, Nino and Zhou, Yongchao and Stanczyk, Piotr and Ramos Garea, Sabela and Geist, Matthieu and Bachem, Olivier},
  booktitle={International Conference on Learning Representations},
  volume={2024},
  pages={21246--21263},
  year={2024}
}

@article{liu2026eopsa,
  title={EOPSA: Efficient On-Policy Self-Distilled Safety Alignment},
  author={Liu, Qirui and Sun, Yichen and Wang, Yan and Mi, Yu and Cao, Wei and Shen, Yue and Chu, Zhixuan and Ren, Kui},
  journal={arXiv preprint arXiv:2609.34519},
  year={2026}
}

@article{andrychowicz2017hindsight,
  title={Hindsight experience replay},
  author={Andrychowicz, Marcin and Wolski, Filip and Ray, Alex and Schneider, Jonas and Fong, Rachel and Welinder, Peter and McGrew, Bob and Tobin, Josh and Pieter Abbeel, OpenAI and Zaremba, Wojciech},
  journal={Advances in neural information processing systems},
  volume={30},
  year={2017}
}

@article{wang2026skill,
  title={Skill-sd: Skill-conditioned self-distillation for multi-turn llm agents},
  author={Wang, Hao and Wang, Guozhi and Xiao, Han and Zhou, Yufeng and Pan, Yue and Wang, Jichao and Xu, Ke and Wen, Yafei and Ruan, Xiaohu and Chen, Xiaoxin and others},
  journal={arXiv preprint arXiv:2604.10674},
  year={2026}
}

@article{wu2026seed,
  title={SEED: Self-Evolving On-Policy Distillation for Agentic Reinforcement Learning},
  author={Wu, Jinyang and Yang, Shuo and Lu, Zhengxi and Zhang, Fan and Shen, Yuhao and Feng, Lang and Luo, Haoran and Lian, Zheng and Zhang, Shuai and Wen, Zhengqi and others},
  journal={arXiv preprint arXiv:2607.14777},
  year={2026}
}

@inproceedings{wu2026atlas,
  title={Atlas: Orchestrating heterogeneous models and tools for multi-domain complex reasoning},
  author={Wu, Jinyang and Zhai, Guocheng and Jin, Ruihan and Yuan, Jiahao and Shen, Yuhao and Zhang, Shuai and Wen, Zhengqi and Tao, Jianhua},
  booktitle={Findings of the Association for Computational Linguistics: ACL 2026},
  pages={17503--17535},
  year={2026}
}

@article{luo2025large,
  title={Large language model agent: A survey on methodology, applications and challenges},
  author={Luo, Junyu and Zhang, Weizhi and Yuan, Ye and Zhao, Yusheng and Yang, Junwei and Gu, Yiyang and Wu, Bohan and Chen, Binqi and Qiao, Ziyue and Long, Qingqing and others},
  journal={arXiv preprint arXiv:2503.21460},
  year={2025}
}

@article{yan2026odysseyarena,
  title={Odysseyarena: Benchmarking large language models for long-horizon, active and inductive interactions},
  author={Yan, Hang and Xu, Fangzhi and Sun, Qiushi and Wu, Jinyang and Huang, Zixian and Huang, Muye and Gong, Jingyang and Ding, Zichen and Cheng, Kanzhi and Wang, Yian and others},
  journal={arXiv preprint arXiv:2602.05843},
  year={2026}
}

@article{shinn2023reflexion,
  title={Reflexion: Language agents with verbal reinforcement learning},
  author={Shinn, Noah and Cassano, Federico and Gopinath, Ashwin and Narasimhan, Karthik and Yao, Shunyu},
  journal={Advances in neural information processing systems},
  volume={36},
  pages={8634--8652},
  year={2023}
}

@article{wang2023voyager,
  title={Voyager: An open-ended embodied agent with large language models},
  author={Wang, Guanzhi and Xie, Yuqi and Jiang, Yunfan and Mandlekar, Ajay and Xiao, Chaowei and Zhu, Yuke and Fan, Linxi and Anandkumar, Anima},
  journal={arXiv preprint arXiv:2305.16291},
  year={2023}
}

@inproceedings{zhao2024expel,
  title={Expel: Llm agents are experiential learners},
  author={Zhao, Andrew and Huang, Daniel and Xu, Quentin and Lin, Matthieu and Liu, Yong-Jin and Huang, Gao},
  booktitle={Proceedings of the AAAI Conference on Artificial Intelligence},
  volume={38},
  number={17},
  pages={19632--19642},
  year={2024}
}

@article{madaan2023self,
  title={Self-refine: Iterative refinement with self-feedback},
  author={Madaan, Aman and Tandon, Niket and Gupta, Prakhar and Hallinan, Skyler and Gao, Luyu and Wiegreffe, Sarah and Alon, Uri and Dziri, Nouha and Prabhumoye, Shrimai and Yang, Yiming and others},
  journal={Advances in neural information processing systems},
  volume={36},
  pages={46534--46594},
  year={2023}
}

@article{yang2026opid,
  title={Opid: On-policy skill distillation for agentic reinforcement learning},
  author={Yang, Shuo and Wu, Jinyang and Lu, Zhengxi and Shen, Yuhao and Zhang, Fan and Feng, Lang and Zhang, Shuai and Luo, Haoran and Lian, Zheng and Wen, Zhengqi and others},
  journal={arXiv preprint arXiv:2606.26790},
  year={2026}
}

@article{wang2026agentopsd,
  title={Agentopsd: Recursive self-distillation for agentic reinforcement learning},
  author={Wang, Zi-Han and Lu, Zhengxi and Yao, Zhiyuan and Wu, Jinyang and Wu, Jie and Cai, Zhengzhou and Sun, Yueqing and Ye, Ziang and Hao, Linji and Gu, Qi and others},
  journal={arXiv preprint arXiv:2608.05987},
  year={2026}
}

@article{lu2026self,
  title={Self-distilled agentic reinforcement learning},
  author={Lu, Zhengxi and Yao, Zhiyuan and Han, Zhuowen and Wang, Zi-Han and Wu, Jinyang and Gu, Qi and Cai, Xunliang and Lu, Weiming and Xiao, Jun and Zhuang, Yueting and others},
  journal={arXiv preprint arXiv:2605.15155},
  year={2026}
}

@article{zheng2026group,
  title={Group-reflective self-distillation for agentic reinforcement learning},
  author={Zheng, Binbin and Xie, Zijun and Zhao, Guanqun and Gong, Enlei and Ma, Xing and Fu, Xiaoliang and Chen, Zeyu},
  journal={arXiv preprint arXiv:2607.28076},
  year={2026}
}

@article{hinton2015distilling,
  title={Distilling the knowledge in a neural network},
  author={Hinton, Geoffrey and Vinyals, Oriol and Dean, Jeff},
  journal={arXiv preprint arXiv:1503.02531},
  year={2015}
}

@inproceedings{kim2016sequence,
  title={Sequence-level knowledge distillation},
  author={Kim, Yoon and Rush, Alexander M},
  booktitle={Proceedings of the 2016 conference on empirical methods in natural language processing},
  pages={1317--1327},
  year={2016}
}

@inproceedings{ross2011reduction,
  title={A reduction of imitation learning and structured prediction to no-regret online learning},
  author={Ross, St{\'e}phane and Gordon, Geoffrey and Bagnell, Drew},
  booktitle={Proceedings of the fourteenth international conference on artificial intelligence and statistics},
  pages={627--635},
  year={2011},
  organization={JMLR Workshop and Conference Proceedings}
}

@article{yang2026self,
  title={Self-distilled rlvr},
  author={Yang, Chenxu and Qin, Chuanyu and Si, Qingyi and Chen, Minghui and Gu, Naibin and Yao, Dingyu and Lin, Zheng and Wang, Weiping and Wang, Jiaqi and Duan, Nan},
  journal={arXiv preprint arXiv:2604.03128},
  year={2026}
}

@article{zhong2026sod,
  title={Sod: Step-wise on-policy distillation for small language model agents},
  author={Zhong, Qiyong and Zheng, Mao and Song, Mingyang and Lin, Xin and Sun, Jie and Jiang, Houcheng and Wang, Xiang and Fang, Junfeng},
  journal={arXiv preprint arXiv:2605.07725},
  year={2026}
}

@misc{zai2026glm52,
  author = {{Z.ai}},
  title  = {{GLM-5.2}: Built for Long-Horizon Tasks},
  year   = {2026},
  month  = jun,
  url    = {https://z.ai/blog/glm-5.2},
  note   = {Accessed: 2026-10-23}
}

@article{aggarwal2025l1,
  title={L1: Controlling how long a reasoning model thinks with reinforcement learning},
  author={Aggarwal, Pranjal and Welleck, Sean},
  journal={arXiv preprint arXiv:2503.04697},
  year={2025}
}

@article{yi2026shorterbetter,
  title={Shorterbetter: Guiding reasoning models to find optimal inference length for efficient reasoning},
  author={Yi, Jingyang and Wang, Jiazheng and Li, Sida},
  journal={Advances in Neural Information Processing Systems},
  volume={38},
  pages={39011--39043},
  year={2026}
}

@article{bradley1952rank,
  title={Rank analysis of incomplete block designs: I. the method of paired comparisons},
  author={Bradley, Ralph Allan and Terry, Milton E},
  journal={Biometrika},
  volume={39},
  number={3/4},
  pages={324--345},
  year={1952},
  publisher={JSTOR}
}

@article{schulman2017proximal,
  title={Proximal policy optimization algorithms},
  author={Schulman, John and Wolski, Filip and Dhariwal, Prafulla and Radford, Alec and Klimov, Oleg},
  journal={arXiv preprint arXiv:1707.06347},
  year={2017}
}

@article{shridhar2020alfworld,
  title={Alfworld: Aligning text and embodied environments for interactive learning},
  author={Shridhar, Mohit and Yuan, Xingdi and C{\^o}t{\'e}, Marc-Alexandre and Bisk, Yonatan and Trischler, Adam and Hausknecht, Matthew},
  journal={arXiv preprint arXiv:2010.03768},
  year={2020}
}

@article{yao2022webshop,
  title={Webshop: Towards scalable real-world web interaction with grounded language agents},
  author={Yao, Shunyu and Chen, Howard and Yang, John and Narasimhan, Karthik},
  journal={Advances in Neural Information Processing Systems},
  volume={35},
  pages={20744--20757},
  year={2022}
}

@misc{hui2024qwen2,
      title={Qwen2.5 Technical Report}, 
      author={Qwen and : and An Yang and Baosong Yang and Beichen Zhang and Binyuan Hui and Bo Zheng and Bowen Yu and Chengyuan Li and Dayiheng Liu and Fei Huang and Haoran Wei and Huan Lin and Jian Yang and Jianhong Tu and Jianwei Zhang and Jianxin Yang and Jiaxi Yang and Jingren Zhou and Junyang Lin and Kai Dang and Keming Lu and Keqin Bao and Kexin Yang and Le Yu and Mei Li and Mingfeng Xue and Pei Zhang and Qin Zhu and Rui Men and Runji Lin and Tianhao Li and Tianyi Tang and Tingyu Xia and Xingzhang Ren and Xuancheng Ren and Yang Fan and Yang Su and Yichang Zhang and Yu Wan and Yuqiong Liu and Zeyu Cui and Zhenru Zhang and Zihan Qiu},
      year={2025},
      eprint={2412.15115},
      archivePrefix={arXiv},
      primaryClass={cs.CL},
      url={https://arxiv.org/abs/2412.15115}, 
}

@article{yang2025qwen3,
  title={Qwen3 technical report},
  author={Yang, An and Li, Anfeng and Yang, Baosong and Zhang, Beichen and Hui, Binyuan and Zheng, Bo and Yu, Bowen and Gao, Chang and Huang, Chengen and Lv, Chenxu and others},
  journal={arXiv preprint arXiv:2505.09388},
  year={2025}
}

@article{xia2026skillrl,
  title={Skillrl: Evolving agents via recursive skill-augmented reinforcement learning},
  author={Xia, Peng and Chen, Jianwen and Wang, Hanyang and Liu, Jiaqi and Zeng, Kaide and Wang, Yu and Han, Siwei and Zhou, Yiyang and Zhao, Xujiang and Chen, Haifeng and others},
  journal={arXiv preprint arXiv:2602.08234},
  year={2026}
}
\endgroup

\clearpage
\appendix
% ============================================================
% Appendix Overview
% Required packages:
% \usepackage[table]{xcolor}
% \usepackage{tabularx}
% \usepackage{booktabs}
% \usepackage{array}
% \usepackage[most]{tcolorbox}
% ============================================================

\definecolor{appblue}{HTML}{315C9B}
\definecolor{apptext}{HTML}{20242B}
\definecolor{appgray}{HTML}{667085}
\definecolor{appline}{HTML}{D9DEE6}
\definecolor{appfill}{HTML}{F4F6F9}

\newcolumntype{Y}{>{\raggedright\arraybackslash}X}

% Main appendix entry
\newcommand{\AppEntry}[3]{%
    \rowcolor{appfill}
    \textcolor{appblue}{\bfseries #1} &
    \textcolor{apptext}{\bfseries #2} &
    \textcolor{appblue}{\bfseries\pageref{#3}} \\[-1pt]
}

% Subsection entry
\newcommand{\AppSubEntry}[3]{%
    \textcolor{appgray}{\hspace{0.9em}\scriptsize\bfseries #1} &
    \textcolor{apptext}{#2} &
    \textcolor{appgray}{\pageref{#3}} \\
}

\newcommand{\AppGap}{%
    \addlinespace[2.5pt]
}

% ============================================================
% Appendix Overview
% ============================================================

\noindent{\LARGE\bfseries\color{sectionblue} Appendix Overview}

\vspace{0.25em}

{\color{appblue}\rule{\textwidth}{1.05pt}}

\vspace{0.4em}

\begin{tcolorbox}[
    enhanced,
    colback=white,
    colframe=appline,
    boxrule=0.5pt,
    arc=1.2pt,
    left=5pt,
    right=5pt,
    top=4pt,
    bottom=4pt,
    boxsep=0pt
]

\small
\renewcommand{\arraystretch}{1.08}
\setlength{\tabcolsep}{5pt}

\begin{tabularx}{\linewidth}{
    @{}
    >{\raggedright\arraybackslash}p{0.17\linewidth}
    Y
    >{\raggedleft\arraybackslash}p{0.055\linewidth}
    @{}
}

% ---------- Header ----------
\textcolor{appgray}{\scriptsize\bfseries SECTION} &
\textcolor{appgray}{\scriptsize\bfseries CONTENTS} &
\textcolor{appgray}{\scriptsize\bfseries PAGE}
\\[-2pt]

\arrayrulecolor{appline}
\midrule
\addlinespace[2pt]

% ============================================================
% Appendix A
\AppEntry{Appendix A}{Theoretical Analysis}{app:theory}
\AppSubEntry{A.1}{Bounded Signed Distillation}{app:bounded_distillation}
\AppSubEntry{A.2}{Influence of Irrelevant Selected Tokens}{app:selection_influence}
\AppSubEntry{A.3}{Properties of Success-Relative Efficiency Rewards}{app:pairwise_properties}
\AppGap

% Appendix B
\AppEntry{Appendix B}{Experimental Details}{app:experimental_details}
\AppSubEntry{B.1}{Training Configuration}{app:training_configuration}
\AppSubEntry{B.2}{Validation Curves Across Backbones}{app:validation_curves}
\AppSubEntry{B.3}{Supervised Data and Hindsight Analysis}{app:sft_preparation}
\AppSubEntry{B.4}{Error-Localization Evaluation Metrics}{app:error_localization_metrics}
\AppSubEntry{B.5}{Motivation Diagnostics}{app:diagnostic_details}
\AppGap

% Appendix C
\AppEntry{Appendix C}{SFT and Analyzer Prompt Templates}{app:analysis_prompts}
\AppSubEntry{C.1}{Outcome-Specific User Template}{app:analysis_prompts:user_template}

\end{tabularx}
\end{tcolorbox}

\section{Theoretical Analysis}
\label{app:theory}

We analyze three properties of SCOPE: bounded signed updates,
the influence of irrelevant distillation tokens, and the
success-relative length reward. The distillation results hold for a
fixed training microbatch and characterize its auxiliary update,
not the convergence of the full RL objective.

\subsection{Bounded Signed Distillation}
\label{app:bounded_distillation}

Let $\mathcal I=\mathcal T_{\mathrm{sel}}$ be the selected token set,
with $n=|\mathcal I|>0$. For each token $k\in\mathcal I$, define
\begin{equation}
\label{eq:theory_notation}
\begin{aligned}
 r_k(\theta)&=\ell_k^T-\ell_k^S(\theta),\\
 v_k(\theta)&=\nabla_\theta\ell_k^S(\theta),\\
 f(r)&=c\tanh(\alpha r/c),\quad \alpha,c>0.
\end{aligned}
\end{equation}
We hold the sampled tokens, selected steps, and teacher scores fixed.
As in Equation~\ref{eq:dist_advantage}, gradients do not pass through
$f(r_k)$. The resulting auxiliary gradient is
\begin{equation}
\label{eq:theory_gradient}
 g(\theta)=-\frac1n\sum_{k\in\mathcal I}f(r_k(\theta))v_k(\theta).
\end{equation}
For $n=0$, we define the auxiliary loss and gradient to be zero.

\begin{scopeproposition}[Log-cosh gradient equivalence]
\label{prop:logcosh}
Define
\begin{equation}
\label{eq:theory_potential}
\begin{aligned}
 \varphi(r)&=\frac{c^2}{\alpha}\log\cosh(\alpha r/c),\\
 \Phi(\theta)&=\frac1n\sum_{k\in\mathcal I}\varphi(r_k(\theta)).
\end{aligned}
\end{equation}
Then $\nabla_\theta\Phi(\theta)=g(\theta)$. The scalar potential
$\varphi$ is nonnegative, symmetric, and strictly convex, with its
unique minimum at $r=0$. Moreover,
\begin{equation}
\label{eq:theory_potential_limits}
\begin{aligned}
 \varphi(r)&=\frac{\alpha}{2}r^2+O(r^4), &&r\to0,\\
 \varphi(r)&=c|r|-\frac{c^2}{\alpha}\log2+o(1),
 &&|r|\to\infty.
\end{aligned}
\end{equation}
\end{scopeproposition}

\begin{proof}
Differentiation gives
\begin{equation}
\label{eq:theory_potential_derivatives}
 \varphi'(r)=f(r),\qquad
 \varphi''(r)=\frac{\alpha}{\cosh^2(\alpha r/c)}>0.
\end{equation}
Since $\ell_k^T$ is fixed, $\nabla_\theta r_k=-v_k$.
Thus $\nabla_\theta\Phi=-\frac1n\sum_k f(r_k)v_k=g$.
Nonnegativity and symmetry follow from $\cosh x\geq1$ and
$\cosh(-x)=\cosh x$. Finally, the expansions
$\log\cosh x=x^2/2+O(x^4)$ near zero and
$\log\cosh x=|x|-\log2+o(1)$ at infinity give
Equation~\ref{eq:theory_potential_limits}.
\end{proof}

This result interprets the detached distillation update as a
log-cosh gradient: small gaps receive approximately linear signed
coefficients, while large gaps have bounded influence.

\begin{scopeproposition}[Bounded auxiliary gradient]
\label{prop:bounded_gradient}
For any finite $r$,
\begin{equation}
\label{eq:theory_coefficient_bounds}
\begin{aligned}
 |f(r)|&<c,\qquad f(0)=0,\\
 \operatorname{sign}f(r)&=\operatorname{sign}r,\\
 0<f'(r)&\leq\alpha.
\end{aligned}
\end{equation}
If $\|v_k\|\leq B_k$ for each selected token, then
\begin{equation}
\label{eq:theory_gradient_bound}
 \|g\|\leq\frac cn\sum_{k\in\mathcal I}B_k.
\end{equation}
In particular, $\|g\|\leq cB$ when $B_k\leq B$ for all $k$.
\end{scopeproposition}

\begin{proof}
The coefficient bounds follow from the range and monotonicity
of $\tanh$, together with
$f'(r)=\alpha/\cosh^2(\alpha r/c)$.
By the triangle inequality,
\[
 \|g\|\leq\frac1n\sum_{k\in\mathcal I}|f(r_k)|\|v_k\|
 \leq\frac cn\sum_{k\in\mathcal I}B_k.
\]
\end{proof}

\begin{scopecorollary}[Sensitivity to teacher scores]
\label{cor:teacher_perturbation}
Suppose the cached teacher score for token $k$ changes by $e_k$,
while the student and selected set remain fixed. If $\|v_k\|\leq B$,
the resulting gradient $\widetilde g$ satisfies
\begin{equation}
\label{eq:theory_teacher_perturbation}
 \|\widetilde g-g\|
 \leq\frac Bn\sum_{k\in\mathcal I}
 \min\{\alpha|e_k|,2c\}.
\end{equation}
\end{scopecorollary}

\begin{proof}
Since $|f'(r)|\leq\alpha$ and $|f(r)|<c$,
$|f(r+e)-f(r)|\leq\min\{\alpha|e|,2c\}$.
Substituting this bound into Equation~\ref{eq:theory_gradient}
and applying the triangle inequality proves the claim.
\end{proof}

Together, these bounds show that $\alpha$ controls local
sensitivity and $c$ limits gap-dependent amplification. They
assume bounded token score gradients and do not imply a global
stability guarantee for RL training.

\subsection{Influence of Irrelevant Selected Tokens}
\label{app:selection_influence}

Step selection removes unselected tokens from the auxiliary loss.
To examine the effect of imperfect selection, partition
$\mathcal I=\mathcal I_{\mathrm{rel}}\cup\mathcal I_{\mathrm{off}}$
into relevant and irrelevant selected tokens, and let
$\eta_{\mathrm{off}}=|\mathcal I_{\mathrm{off}}|/n$.

\begin{scopeproposition}[Off-target gradient contribution]
\label{prop:selection_influence}
Assume $\|v_k\|\leq B$. The gradient contribution from irrelevant
tokens,
\[
 g_{\mathrm{off}}=-\frac1n\sum_{k\in\mathcal I_{\mathrm{off}}}
 f(r_k)v_k,
\]
satisfies
\begin{equation}
\label{eq:theory_offtarget_bound}
 \|g_{\mathrm{off}}\|\leq cB\eta_{\mathrm{off}}.
\end{equation}
If $\mathcal I_{\mathrm{rel}}\neq\varnothing$, let
$\overline g_{\mathrm{rel}}$ be the mean auxiliary gradient over
relevant tokens alone. Then
\begin{equation}
\label{eq:theory_selection_deviation}
 \|g-\overline g_{\mathrm{rel}}\|
 \leq2cB\eta_{\mathrm{off}}.
\end{equation}
\end{scopeproposition}

\begin{proof}
Each token contributes a vector of norm at most $cB$ before
normalization, proving Equation~\ref{eq:theory_offtarget_bound}.
Moreover,
\[
 g=(1-\eta_{\mathrm{off}})\overline g_{\mathrm{rel}}
   +g_{\mathrm{off}}.
\]
Since $\|\overline g_{\mathrm{rel}}\|\leq cB$, the triangle
inequality gives
$\|g-\overline g_{\mathrm{rel}}\|
\leq\eta_{\mathrm{off}}cB+cB\eta_{\mathrm{off}}$.
\end{proof}

Thus, reducing the fraction of irrelevant selected tokens tightens
a worst-case bound on their effect. This result quantifies the role
of selection but does not assume that the analyzer is always correct.

\subsection{Properties of Success-Relative Efficiency Rewards}
\label{app:pairwise_properties}

Consider a task group with $M=|\mathcal S_q|$ successful trajectories
and positive interaction lengths $T_i$. The pairwise score $b_i$
is defined in Equation~\ref{eq:pairwise_score}.

\begin{scopeproposition}[Pairwise reward properties]
\label{prop:pairwise_properties}
For $M\geq2$ and $\tau_{\mathrm{len}}>0$, the scores satisfy:
\begin{enumerate}[label=(\roman*),leftmargin=*]
\item \emph{Length ordering:} $0<b_i<1$, and
$T_i<T_j$ implies $b_i>b_j$; equal lengths give equal scores.
\item \emph{Scale invariance:} replacing all $T_i$ with $aT_i$
for any $a>0$ leaves every $b_i$ unchanged.
\item \emph{Mean-bonus conservation:}
\begin{equation}
\label{eq:theory_bonus_conservation}
\begin{aligned}
 \sum_{i\in\mathcal S_q}b_i&=\frac M2,\\
 \sum_{i\in\mathcal S_q}(R_i^{\mathrm{pair}}-R_i)
 &=\frac{\lambda_{\mathrm{len}}M}{2}.
\end{aligned}
\end{equation}
\end{enumerate}
The last identities also hold for $M=1$ and $M=0$ under the
conventions in Equation~\ref{eq:pairwise_score}.
\end{scopeproposition}

\begin{proof}
For distinct successes $i,j$, Equation~\ref{eq:pairwise_length}
gives $p_{ij}>p_{ji}$ whenever $T_i<T_j$. The sigmoid is strictly
monotone, so $p_{ik}>p_{jk}$ for each other success $k$. Hence
\begin{equation}
\label{eq:theory_pairwise_order}
\begin{aligned}
 (M-1)(b_i-b_j)
 ={}&p_{ij}-p_{ji}\\
 &+\sum_{k\in\mathcal S_q\setminus\{i,j\}}
       (p_{ik}-p_{jk})>0.
\end{aligned}
\end{equation}
Equal lengths give equality, and $0<p_{ij}<1$ implies
$0<b_i<1$. This proves (i). For (ii),
$\log(aT_j)-\log(aT_i)=\log T_j-\log T_i$.
Finally, $p_{ij}+p_{ji}=1$, so
\[
 \sum_{i\in\mathcal S_q}b_i
 =\frac{1}{M-1}\sum_{i<j}(p_{ij}+p_{ji})
 =\frac M2.
\]
Multiplying by $\lambda_{\mathrm{len}}$ proves (iii).
The cases $M=0,1$ follow from the definitions.
\end{proof}

\begin{scopecorollary}[Success--failure reward ordering]
\label{cor:success_order}
If $R_i\geq R_j$ for every successful trajectory $i$ and failed
trajectory $j$, then for $\lambda_{\mathrm{len}}>0$,
$R_i^{\mathrm{pair}}>R_j^{\mathrm{pair}}$.
Among successes with equal original rewards, shorter trajectories
receive larger shaped rewards.
\end{scopecorollary}

\begin{proof}
For a success $i$, $b_i>0$, whereas a failure receives no
length bonus. Therefore,
\[
 R_i^{\mathrm{pair}}-R_j^{\mathrm{pair}}
 =(R_i-R_j)+\lambda_{\mathrm{len}}b_i>0.
\]
The ordering among equally rewarded successes follows from
Proposition~\ref{prop:pairwise_properties}(i).
\end{proof}

The pairwise reward therefore favors shorter successes
without a fixed length target. Its mean bonus is always
$\lambda_{\mathrm{len}}/2$ across successes, while the relative
allocation changes with the current rollout group.

\section{Experimental Details}
\label{app:experimental_details}

\subsection{Training Configuration}
\label{app:training_configuration}

SCOPE follows the supervised-initialization and online-training
workflow of SEED~\citep{wu2026seed}. Table~\ref{tab:scope_configuration}
summarizes the key coefficients used by the full method. The policy
snapshot supplies both the actor and the online analyzer; its
skill-conditioned teacher scores are cached before student
optimization. Online analysis has no separate supervised loss:
the policy is updated through the joint objective in
Equation~\ref{eq:joint_objective}. For the smaller backbones, training uses two NVIDIA
H200 GPUs on WebShop and four on ALFWorld.
The Qwen3-32B experiments use 16 NVIDIA H200 GPUs.

\begin{table}[htbp]
\centering
\caption{Key settings of full SCOPE. Parameter sweeps vary one of
$\alpha$, $c$, or $\tau_{\mathrm{len}}$ at a time.}
\label{tab:scope_configuration}
\small
\setlength{\tabcolsep}{4pt}
\renewcommand{\arraystretch}{1.12}
\begin{tabular}{@{}lc@{}}
\toprule
\textbf{Parameter} & \textbf{Value} \\
\midrule
Maximum skill entries $K_{\max}$ & 5 \\
Signed-gap slope $\alpha$ & 2.5 \\
Saturation bound $c$ & 1 \\
Distillation weight $\lambda$ & 0.01 \\
Length-preference temperature $\tau_{\mathrm{len}}$ & 0.1 \\
Length-bonus weight $\lambda_{\mathrm{len}}$ & 5 \\
PPO clipping radius $\epsilon$ & 0.2 \\
Reference KL weight $\beta_{\mathrm{ref}}$ & 0.01 \\
Entropy weight $\lambda_H$ & 0.001 \\
\bottomrule
\end{tabular}
\end{table}

The main comparison uses the step-150 checkpoint. ALFWorld ablations
and parameter sweeps use step 180. Training-step savings in
Figure~\ref{fig:intro_overview}(c) compare the first observed
checkpoints reaching GRPO's final mean success rate.

\subsection{Validation Curves Across Backbones}
\label{app:validation_curves}

To complement the multi-seed training dynamics in
Figure~\ref{fig:intro_overview}(c), we report single-run
validation curves for SCOPE with Qwen2.5-3B-Instruct
and Qwen2.5-7B-Instruct on ALFWorld and WebShop.
Both figures mark training step 150, the checkpoint
used for the main comparison in
Table~\ref{tab:main_results}.

\begin{figure}[tbp]
    \centering
    \includegraphics[width=\linewidth]
    {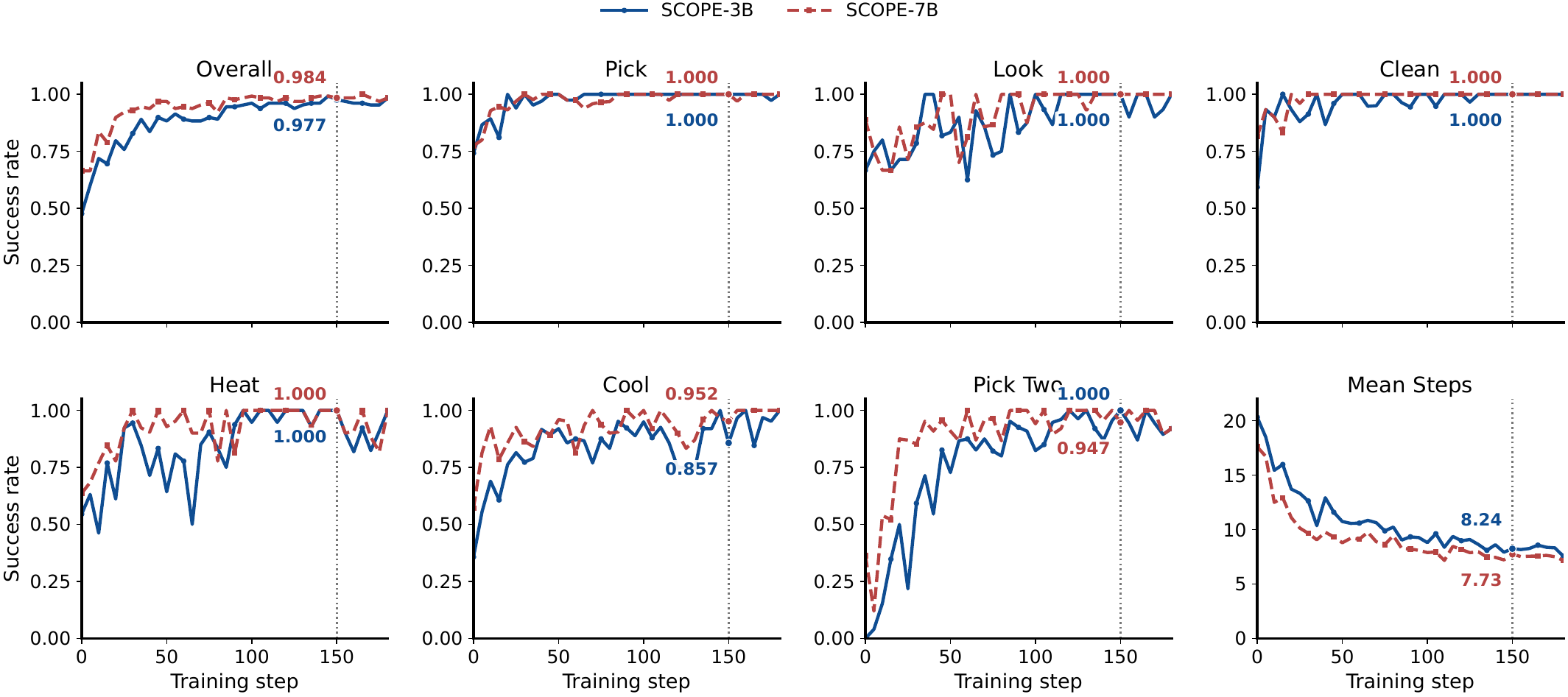}
    \caption{\textbf{Single-run validation curves on ALFWorld.}
    Overall and task-family success rates, together with
    mean environment interaction steps, for SCOPE-3B
    and SCOPE-7B. The dotted vertical line indicates
    step 150, whose overall success rates match
    Table~\ref{tab:main_results}.}
    \label{fig:app_alfworld_validation_curves}
\end{figure}

\begin{figure}[tbp]
    \centering
    \includegraphics[width=\linewidth]
    {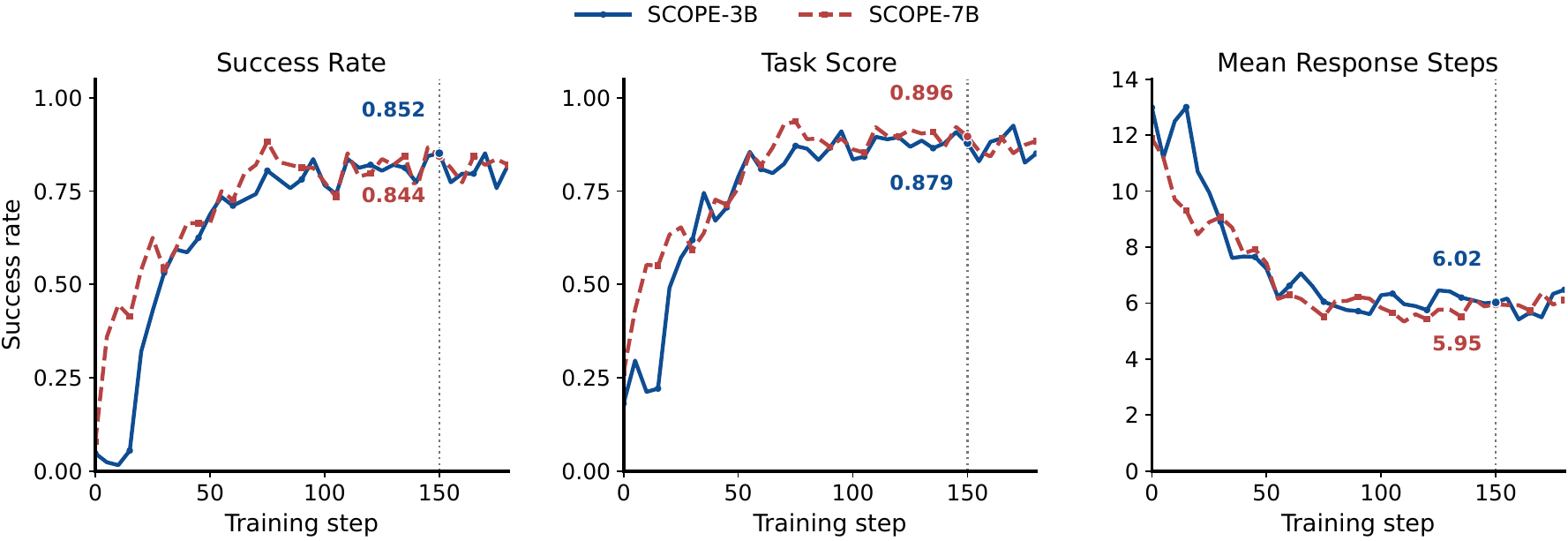}
    \caption{\textbf{Single-run validation curves on WebShop.}
    Success rate, normalized task score, and mean
    environment interaction steps for SCOPE-3B
    and SCOPE-7B. The dotted vertical line indicates
    step 150; the success rates and normalized scores
    match Table~\ref{tab:main_results}.}
    \label{fig:app_webshop_validation_curves}
\end{figure}

As shown in Figures~\ref{fig:app_alfworld_validation_curves}
and~\ref{fig:app_webshop_validation_curves}, both backbones
improve substantially during online training while reducing
interaction steps. At step 150, SCOPE-3B and SCOPE-7B
reach 97.7\% and 98.4\% overall ALFWorld success,
respectively. On WebShop, they achieve 85.2\% and
84.4\% success, with normalized task scores of
87.9\% and 89.6\%, respectively.
The corresponding mean interaction counts are
8.24 and 7.73 on ALFWorld, and 6.02 and 5.95
on WebShop. These results provide single-run
training context for the main-table checkpoint values.

\subsection{Supervised Data and Hindsight Analysis}
\label{app:sft_preparation}

GLM-5.2 generates the offline annotations used for SFT. Each input
contains the task description, episode outcome, and the recorded
sequence of observations and responses, together with action-validity
feedback when available. Successful trajectories are analyzed for
avoidable errors. Failed trajectories may additionally receive up
to three successful action sequences from the same task as reference
context. These offline references can come from separate runs of that
task; online references are drawn from the current on-policy group.

The annotation configuration uses temperature 0, an output limit of
8,192 tokens, and disabled thinking. Observations are retained without
character-level truncation. An annotation contains a trajectory
summary and at most five skill entries, each with target-step indices,
an error explanation, and a short reusable corrective instruction.
Entries that require the same correction can cover multiple steps.
The prompt asks the annotator to judge decisions using the information
available at each step and to preserve reasonable exploration. An
empty skill list is returned when no clear error is supported.

The data pipeline checks the JSON structure, nonempty required fields,
and valid, unique target indices before export. Each SFT example is a
system--user--assistant conversation: the input is the analysis prompt,
and the assistant target contains only the structured summary and skill
entries. Successful reference action lists remain in the input rather
than in the prediction target. The same analysis instructions and
output schema are used by the online analyzer, which is the current
policy snapshot. If a failed group contains no success, the reference
block is omitted and the analyzer judges the recorded trajectory
directly. Appendix~\ref{app:analysis_prompts} provides the templates.

\subsection{Error-Localization Evaluation Metrics}
\label{app:error_localization_metrics}

We evaluate error-step localization against reference annotations.
For each trajectory $i$, let $E_i$ and $\widehat E_i$ denote
the reference and predicted error-step sets, respectively.
Predictions are extracted from the \texttt{target\_steps}
fields of generated skill entries.

\paragraph{Step-level metrics}
Micro-precision and micro-recall are defined as
$P=\frac{\mathrm{TP}}{\mathrm{TP}+\mathrm{FP}}$ and
$R=\frac{\mathrm{TP}}{\mathrm{TP}+\mathrm{FN}}$,
with counts aggregated across all trajectories.
Micro-F1 is $\frac{2PR}{P+R}$, while Macro-F1 averages
the F1 scores computed separately for each trajectory.

\paragraph{Trajectory-level metrics}
Exact Match measures the fraction of trajectories with
$\widehat E_i=E_i$. Parse Rate measures the fraction of
outputs satisfying the required JSON schema.
All metrics are reported as percentages.

\subsection{Motivation Diagnostics}
\label{app:diagnostic_details}

\paragraph{Outcome--step mismatch}
The WebShop study in Figure~\ref{fig:intro_overview}(a) uses
GLM-5.3 to identify erroneous steps in 1,400 trajectories: 536
successful and 864 failed episodes. Local decision quality is judged
separately from episode outcome: an eventual success can contain
avoidable mistakes, and an eventual failure can contain correct
decisions. The figure reports error and correct-step proportions
separately for the two outcome groups. GLM-5.2 supplies the offline
SFT annotations described above; GLM-5.3 supplies these diagnostic
error-step labels.

\paragraph{Off-target probability shifts}
The skill-influence analysis scores 33,602 steps and 4,035,655
response tokens from 177 task groups. The recorded relevance labels
divide steps into 14,081 related and 19,521 unrelated steps according
to the applicability of the trajectory-level skill. On the sampled
response tokens, we compute the signed gap
$\Delta=\log\pi_T(y\mid h,s)-\log\pi_S(y\mid h)$.
Within each outcome and relevance group, positive and negative bars
are pooled-token means conditional on $\Delta>0$ and $\Delta<0$,
respectively. Zero gaps are excluded from these conditional means.
The 57.9\% off-target share divides the number of nonzero-gap tokens
at unrelated steps by the number of nonzero-gap tokens at all steps.
The negative and positive shares of 53.9\% and 46.1\% use nonzero-gap
tokens at unrelated steps as their denominator.

Error bars show 95\% task-group bootstrap intervals from 2,000
resamples with random seed 0. Task descriptions and initial
observations define the grouping proxy; each resample draws complete
groups rather than individual tokens. This preserves within-group
dependence when estimating uncertainty in the conditional means.

\section{SFT and Analyzer Prompt Templates}
\label{app:analysis_prompts}

The offline annotator and online analyzer share the system message,
analysis instructions, and JSON schema below. A user message consists
of the appropriate outcome-specific opening, the shared instructions,
and the episode context. Placeholders in angle brackets are replaced
with the recorded task and trajectory. The reference block is included
only for failures with available successful references. In the
serialized trajectory, each step contains its index, observation,
response, and action-validity flag when recorded.

\begin{scopepromptbox}{System message}{scopepromptblue}
\begin{Verbatim}[fontsize=\footnotesize,breaklines=true,breakanywhere=true,
  breaksymbolleft={},breaksymbolright={},
  breakanywheresymbolpre={},breakanywheresymbolpost={}]
You analyze agent episodes and extract concise, policy-facing skills for specific error steps. Return ONLY valid JSON.
\end{Verbatim}
\end{scopepromptbox}
\begin{scopepromptbox}{Shared analysis instructions and output schema}{scopepromptteal}
\begin{Verbatim}[fontsize=\footnotesize,breaklines=true,breakanywhere=true,
  breaksymbolleft={},breaksymbolright={},
  breakanywheresymbolpre={},breakanywheresymbolpost={}]
First write a concise trajectory_summary. Then return up to 5 step_skills. For each:
1. Identify clear error steps in target_steps; steps requiring the same correction share one skill.
2. Write one short policy-facing skill that extracts the correction into a reusable workflow: the core decision rule and action ordering that would avoid the mistake.
3. Write a one-sentence error_reason that cites the erroneous steps and explains why the actions at those steps were mistakes.

Important constraints:
- Judge errors using information available at each target step; do not label reasonable exploration.
- Write skill as one short, natural instruction, not a retrospective explanation or a list of labeled sections.
- Generalize beyond this task without names, IDs, exact values, or step numbers; keep the advice concrete and applicable at every target step.
- Use existing step indices, sorted and unique; sort entries by first target step.
- If no clear error is supported, still write trajectory_summary and return {"step_skills":[]}.
- Return only the top-level fields trajectory_summary and step_skills.

Return format (replace string placeholders and 0 with actual text and existing integer step indices):
{"trajectory_summary":"string","step_skills":[{"target_steps":[0],"error_reason":"string","skill":"string"}]}
\end{Verbatim}
\end{scopepromptbox}

\subsection{Outcome-Specific User Template}
\label{app:analysis_prompts:user_template}

The same template is used for successful and failed trajectories.
Outcome-specific instructions are selected accordingly, while
successful references are included only for failures when available.

\begin{scopepromptbox}{User template: successful / failed trajectory}{scopepromptpurple}
\begin{Verbatim}[fontsize=\footnotesize,breaklines=true,breakanywhere=true,
  breaksymbolleft={},breaksymbolright={},
  breakanywheresymbolpre={},breakanywheresymbolpost={}]
Analyze the following <SUCCESSFUL/FAILED> trajectory.
<IF SUCCESSFUL: It may still contain avoidable mistakes.>
<IF FAILED: Successful references are hints; differences alone do not prove mistakes.>
<SHARED ANALYSIS INSTRUCTIONS AND OUTPUT SCHEMA>

Episode context:
- Task description: <TASK>
- episode_success: <success/failure>
- Interaction trajectory:
Step <STEP_INDEX>
Observation: <OBSERVATION>
Response: <RESPONSE>
Action valid: <true_or_false>

<REMAINING RECORDED STEPS>

<IF FAILED AND REFERENCES AVAILABLE:
Successful reference trajectories (actions only):
Successful reference 1:
Action 0: <ACTION>
Action 1: <ACTION>
<REMAINING REFERENCE ACTIONS>
<ADDITIONAL SUCCESSFUL REFERENCES, IF AVAILABLE>
>
\end{Verbatim}
\end{scopepromptbox}

\end{document}